\def\WITHAPPENDIX{1}
\documentclass[letterpaper]{article} 
\usepackage[preprint]{aaai2027}  
\usepackage[hyphens]{url}  
\usepackage{graphicx} 
\usepackage{natbib}  
\usepackage{caption} 
\usepackage{amsmath,amssymb,amsthm,bm}
\usepackage{booktabs,multirow,tabularx,array}
\usepackage{algorithm}
\usepackage{algorithmic}

\newtheorem{proposition}{Proposition}
\newtheorem{lemma}{Lemma}
\newtheorem{remark}{Remark}
\theoremstyle{definition}

\newtheorem{assumption}{Assumption}

\newcommand{\bcppo}{\textnormal{\textsc{BCPPO}}}
\newcommand{\E}{\mathbb{E}}
\renewcommand{\Pr}{\mathbb{P}}
\DeclareMathOperator{\VaR}{VaR}
\DeclareMathOperator{\CVaR}{CVaR}
\newcommand{\Clim}{C_{\mathrm{lim}}}
\newcommand{\RB}{\mathcal{R}_B}
\newcommand{\ABP}{A_{\mathrm{BP}}}
\newcommand{\norm}[1]{\mathrm{norm}(#1)}
\newcommand{\stopgrad}[1]{\overline{#1}}

\ifdefined\WITHAPPENDIX
  \newcommand{\suppref}[2]{#1~\ref{#2}}
  \newcommand{\Suppref}[2]{#1~\ref{#2}}
  
\else
  \newcommand{\suppref}[2]{supplementary material}
  \newcommand{\Suppref}[2]{Supplementary material}
  
\fi

\title{%
  \bcppo: Bachelier-Inspired Constrained Proximal Policy Optimization\\
  for Tail-Risk-Aware Safe Reinforcement Learning
}
\author{
  Dongsheng Hou\equalcontrib \quad
  Yanqiao Chen\equalcontrib \quad
  Yuhan Rui\equalcontrib
}
\affiliations{
  Southern University of Science and Technology\\
  \texttt{\{12410421,12412115\}@mail.sustech.edu.cn},
  \texttt{ruiyuhan0110@gmail.com}
}
\date{}

\begin{document}
\maketitle

\begin{abstract}
Expected-cost constraints can still permit rare, high-cost events. Monte Carlo
conditional value at risk (CVaR) gradients can be noisy at high confidence,
whereas critics that model an outcome distribution add complexity. We propose
\bcppo{} (Bachelier-Inspired Constrained Proximal Policy Optimization), a
proximal policy optimization (PPO) method. Separately initialized cost-prediction
networks (critics), trained with random sample masks, produce disagreement that
marks predictions sensitive to which state--action regions occur in the training
data and to critic training. A Bachelier
formula for the expected amount above a reference level converts this
disagreement into a smooth policy-update penalty. Gradients from this penalty do
not alter the critics, so temporal-difference (TD) critic learning is unchanged.
A saturation-aware controller adjusts the mean-cost penalty and stops
accumulated error from growing while that penalty is clipped. Deployment retains
only the policy network. The disagreement penalty is neither a tail-event
probability nor a guaranteed error bound, and it provides no safety guarantee.
Across 175 runs with
shared tasks, costs, budgets, training steps, and evaluation seeds, no comparator
attains both higher mean return and lower mean CVaR than \bcppo{} in any task.
On Push1, \bcppo{} has no lower return and no higher CVaR than every comparator,
with at least one strict gain. These results support a practical balance among
reward, caution around cost predictions that vary across trained critics, and
policy-only deployment.
\end{abstract}

\section{Introduction}
\label{sec:intro}

Deploying reinforcement learning (RL) in safety-critical settings---autonomous
driving, robot manipulation, medical treatment planning---demands more than
reward maximization: it demands that harmful outcomes be rare.
The constrained Markov decision process
(CMDP)~\citep{altman1999} provides the standard framework,
constraining the \emph{expected} discounted cumulative cost:
\begin{align}
  \max_{\pi}\; J_r(\pi) &= \E_{\tau\sim\pi}\!\left[\sum_{t=0}^{\infty}
    \gamma^t r(s_t,a_t)\right] \label{eq:reward_obj}\\
  \text{s.t.}\quad J_c(\pi) &= \E_{\tau\sim\pi}\!\left[\sum_{t=0}^{\infty}
    \gamma^t c(s_t,a_t)\right] \leq \Clim. \label{eq:cost_constraint}
\end{align}
Here $\pi$ is the policy, $\tau=(s_0,a_0,s_1,\ldots)$ is a trajectory sampled
from it, and $s_t,a_t,r,c,\gamma$, and $\Clim$ denote state, action, reward,
nonnegative safety cost, discount, and expected-cost budget. Writing
$C_\gamma(\tau)=\sum_t\gamma^t c(s_t,a_t)$ and letting $C_{\mathrm{harm}}$ be a
harmful-cost threshold, Eq.~(\ref{eq:cost_constraint}) does not control the tail:
\[
  J_c(\pi)\leq\Clim \;\nRightarrow\;
  \Pr\!\left(C_\gamma(\tau)>C_{\mathrm{harm}}\right) \approx 0.
\]
A policy with zero cost in $95\%$ of trajectories and cost $L$ otherwise still
satisfies the budget whenever $0.05L\leq\Clim$.

Rare events also leave some state--action costs weakly supported by data.
Existing methods emphasize sampled high-cost trajectories, distributional
quantiles, epistemic uncertainty from limited knowledge or model disagreement,
or aleatoric uncertainty from irreducible outcome randomness. We instead retain
mean-cost control while marking predictions sensitive to critic training, without
interpreting that sensitivity as a tail probability.

To address this need, we propose \bcppo{}. Its contributions are:
\begin{itemize}
	  \item \textbf{Mean--disagreement decomposition.}
	    Cost critics trained with independent random sample masks separate the mean
	    prediction used by the standard cost branch from cross-critic disagreement.
	  \item \textbf{Bachelier-inspired policy-update shaping.}
	    A state--action-dependent reference level makes the expected-excess penalty
	    equal to a positive coefficient computed directly from the fixed risk
	    parameters times ensemble spread, without
	    counting mean cost twice. Blocking gradients from this penalty to the
	    critics preserves critic learning; mean-cost
	    control stops accumulated error from growing while its multiplier is clipped;
	    and deployment uses only the policy.
	  \item \textbf{Shared-protocol performance and mechanism evidence.}
    Across 175 runs, no comparator improves both mean reward and CVaR over
    \bcppo{} in any task, and \bcppo{} improves at least one without worsening
    the other against every comparator on Push1. Ablations test component effects;
    controlled critic retraining after deliberately removing a region of training data
    tests whether missing coverage raises disagreement beyond effects from
    sample count and artificial noise added to TD targets.
\end{itemize}

\begin{figure*}[t]
\centering
\includegraphics[width=0.48\textwidth]{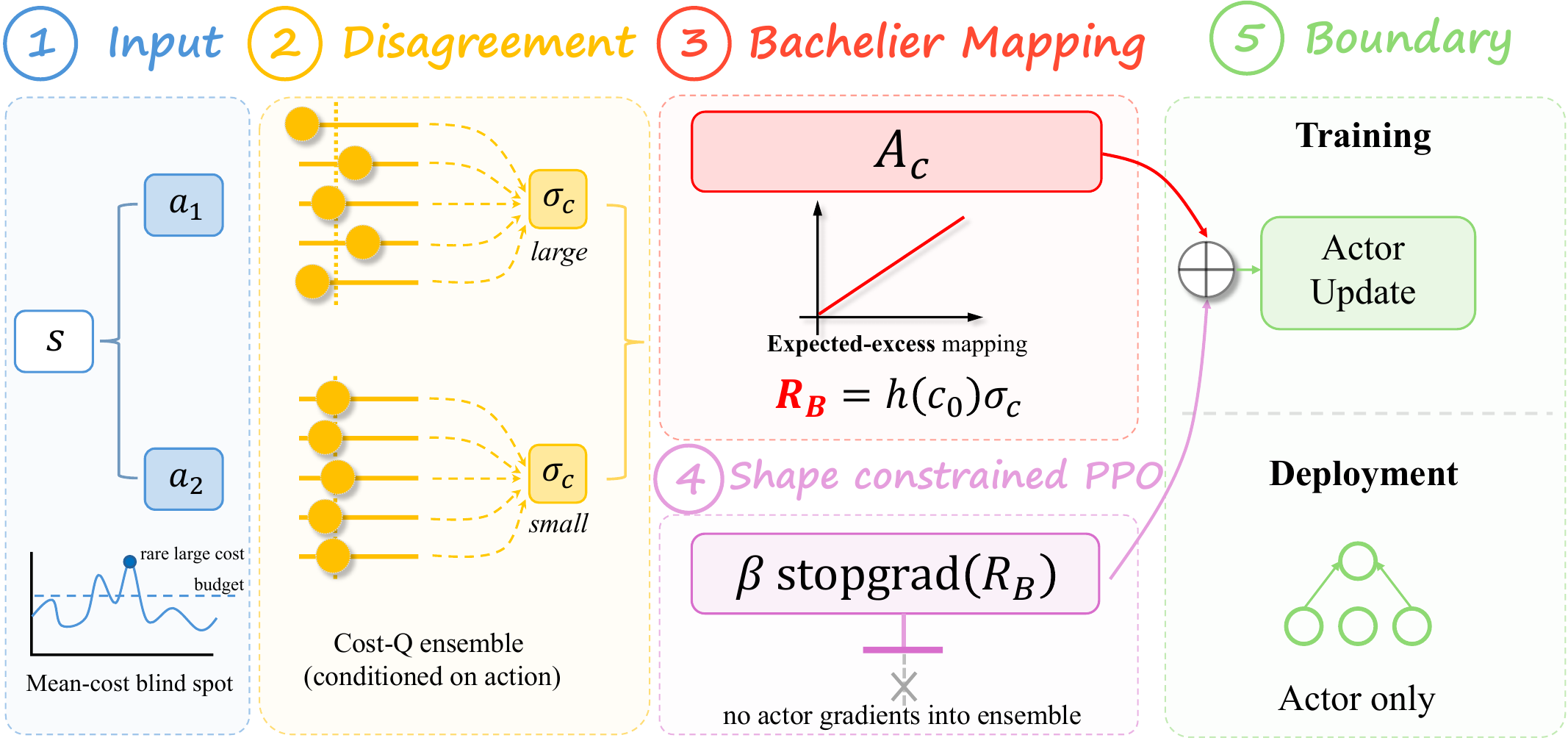}%
\hfill
\includegraphics[width=0.48\textwidth]{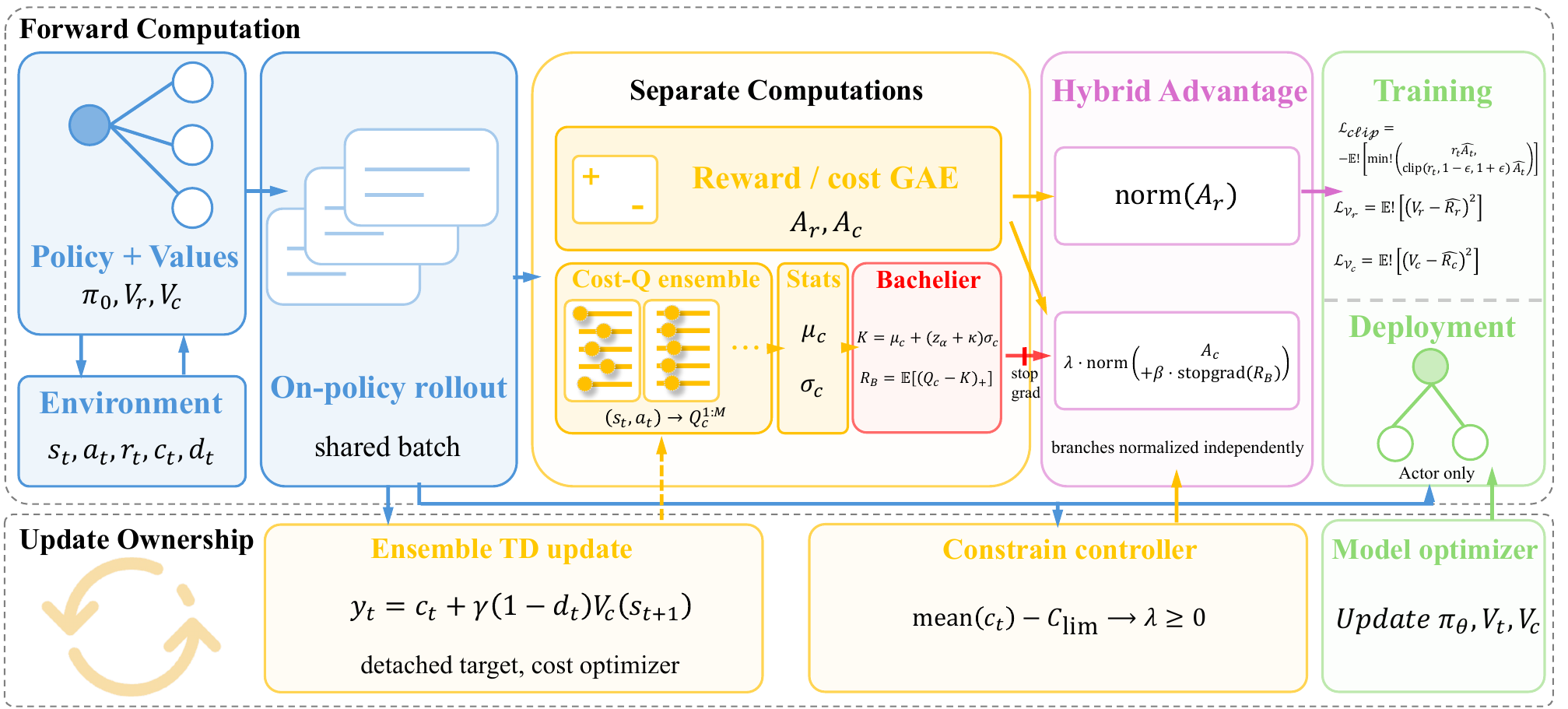}
\makebox[0.48\textwidth][c]{\small (a) Motivation and disagreement-based cost shaping.}%
\hfill
\makebox[0.48\textwidth][c]{\small (b) BCPPO training and deployment data flow.}
\caption{
  BCPPO overview. (a) Mean-cost control can miss rare high-cost outcomes.
  Action-conditioned cost-critic disagreement is mapped to the
  Bachelier-inspired penalty $\RB$ and, with gradients blocked from the
  ensemble, added to the cost branch. The standard reward branch, multiplier,
  and branch normalization are omitted from panel (a) for clarity.
  (b) Shared on-policy rollouts feed reward/cost advantage estimation and a
  separately optimized cost-Q ensemble. Normalized reward and shaped-cost
  branches form $A_{\mathrm{BP}}$ for the PPO actor update; the ensemble receives
  only temporal-difference gradients, and deployment retains only the actor.
}
\label{fig:overview}
\end{figure*}

\section{Related Work}
\label{sec:related}

\paragraph{Expectation-constrained safe reinforcement learning baselines.}
The CMDP framework~\citep{altman1999} underpins methods that constrain expected
cost. Our maintained OmniSafe baselines are Constrained Policy Optimization
(CPO)~\citep{achiam2017constrainedpolicyoptimization}, PPO-Lagrangian
(PPOLag)~\citep{ray2019safety}, and First-Order Constrained Optimization in
Policy Space (FOCOPS)~\citep{zhang2020orderconstrainedoptimizationpolicy}.
They use different updates to enforce the same mean-cost objective and are not
tail-risk algorithms. \bcppo{} adopts
proportional--integral--derivative (PID) Lagrangian
control~\citep{stooke2020responsivesafetyreinforcementlearning} with anti-windup,
which stops integral accumulation when the projected multiplier saturates;
related gradient-based variants also constrain expected cost
\citep{yang2020projectionbasedconstrainedpolicyoptimization,
liu2022constrainedvariationalpolicyoptimization,
sootla2022sauterlsurelysafe}. Such constraints can improve tails indirectly but
do not specify upper-tail behavior.

\paragraph{Sample-based and distributional tail risk.}
Utility-based risk-sensitive RL predates direct CVaR policy
gradients~\citep{mihatsch2002}. Such gradients identify high-cost trajectories
through empirical tail samples and
can become noisier in proportion to
$(1-\alpha)^{-1}$~\citep{tamar2014optimizingcvarsampling}. CVaR-constrained
dynamic programming and PPO provide direct alternatives
\citep{chow2017riskconstrainedreinforcementlearningpercentile,
ying2022towardssafereinforcementlearningcvar}. Our CVaR-weighted PPO comparison
(CPPO) is a controlled Monte Carlo upper-tail comparison, not an external
reference result. Distributional RL
\citep{bellemare2017distributionalperspectivereinforcementlearning,dabney2017distributionalreinforcementlearningquantile,dabney2018implicitquantilenetworksdistributional}
models return variation explicitly, and distributional safety critics place
implicit quantile network cost-return models inside CVaR constraints
\citep{yang2022distributional}. Our quantile-regression PPO (QR PPO) comparison
tests a quantile cost critic in place of uncertainty shaping. Worst-Case Soft
Actor Critic (WCSAC), a Soft Actor-Critic (SAC) method that reuses stored
transitions,
\citep{Yang_Simao_Tindemans_Spaan_2021} supplies the external tail-aware
comparison through Gaussian cost-return modeling and a CVaR
objective. Together these methods separate sampled-tail, distributional-tail,
and mean-cost mechanisms.

\paragraph{Epistemic uncertainty in safe RL.}
Bayesian uncertainty decomposition~\citep{depeweg2018decompositionuncertaintybayesiandeep},
deep ensembles~\citep{lakshminarayanan2017simplescalablepredictiveuncertainty},
and probabilistic ensembles with trajectory sampling
\citep{chua2018deepreinforcementlearninghandful} provide practical
epistemic-uncertainty estimators. Ensemble Q-functions have also supported
offline out-of-distribution (OOD) conservatism and exploration guided by
optimistic uncertainty bounds
\citep{an2021uncertaintybasedofflinereinforcementlearning,
lee2021sunrisesimpleunifiedframework}. Safe-RL methods use ensembles in
different roles:
Bayesian CPO propagates uncertainty through a learned transition model
\citep{as2022constrainedpolicyoptimizationbayesian};
\citet{stachowicz2024racerepistemicrisksensitiverl} apply CVaR to an ensemble
of distributional return critics, jointly exposing the actor to aleatoric
return variation and epistemic disagreement; and Uncertainty-Aware Safety
Propagation Critics (USPC)~\citep{demiray2026uspc}
forms a conservative cost from a scalar ensemble's mean and spread, then
propagates safety labels from actions judged safe to nearby actions.

USPC and \bcppo{} therefore place uncertainty differently. \bcppo{} keeps
unmodified TD targets and the original cost advantage; its reference level
cancels the ensemble mean from $\RB$, leaving critic disagreement as an
actor-side penalty while PID controls mean cost. It neither replaces the cost
estimate with an upper bound nor transfers safety labels across actions.
The disagreement signal is used only during training; its numerical value is
not an OOD probability or a strict tail-risk guarantee.

\section{\bcppo{}: Method}
\label{sec:method}

\bcppo{} separates mean predicted cost from disagreement across trained critics
(Figure~\ref{fig:overview}). Disagreement enters the PPO-Lagrangian actor as a
cautionary penalty; agreement does not certify correctness or safety. A
Bachelier expected-excess formula supplies the penalty coefficient, actor
gradients are blocked from the critics, and anti-windup PID updates the
mean-cost multiplier from observed cost error.

\subsection{CMDP, CVaR, and the Bachelier Formula}

We use the CMDP variables defined in
Eqs.~(\ref{eq:reward_obj})--(\ref{eq:cost_constraint}). In experiments, the PID
error and evaluation cost are length-normalized per-step rates; experimental
$\Clim$ is therefore a rate limit rather than the unnormalized discounted sum.
For a random variable $X$ and confidence level
$\alpha\in(0,1)$, its value at risk (VaR) is
$\VaR_\alpha(X)=\inf\{z:\Pr(X\leq z)\geq\alpha\}$, and its conditional value at
risk (CVaR) is
\begin{equation}
  \CVaR_\alpha(X)=\inf_{\eta\in\mathbb R}
  \left\{\eta+\frac{\E[(X-\eta)_+]}{1-\alpha}\right\}.
  \label{eq:cvar_general}
\end{equation}
Here $(x)_+=\max(x,0)$. For a continuous distribution, CVaR is the mean above
VaR: $\E[X\mid X\geq\VaR_\alpha(X)]$~\citep{tamar2015policygradientcoherentrisk}.
Hereafter, $\Phi$ and $\phi$ denote the cumulative distribution function and
probability density function of a standard normal variable.
For Gaussian $X\sim\mathcal{N}(\mu,\sigma^2)$,
$\VaR_\alpha=\mu+\Phi^{-1}(\alpha)\sigma$ and
$\CVaR_\alpha=\mu+\sigma\phi(\Phi^{-1}(\alpha))/(1-\alpha)$.
For $X\sim\mathcal{N}(\mu,\sigma^2)$ and reference level $K\in\mathbb{R}$,
the Bachelier expected positive excess~\citep{bachelier1900} is
\begin{equation}
  \begin{aligned}
  \mathrm{Call}(K;\mu,\sigma)
  &= \E[(X-K)_+] \\
  &= (\mu-K)\Phi(d)+\sigma\phi(d),\qquad
  d=\frac{\mu-K}{\sigma}.
  \end{aligned}
  \label{eq:bachelier}
\end{equation}
\Suppref{Appendix}{app:bachelier_derivation} gives its derivation, derivatives,
and financial interpretation.

\subsection{Critic Disagreement and Bachelier-Inspired Shaping}

For a fixed state--action pair, separately initialized critics trained with
bootstrap masks can produce different predictions of future cumulative cost.
Their disagreement measures sensitivity to training data and critic optimization. It can
reflect limited coverage or unstable training, but neither proves that an action
is unsafe nor captures aleatoric uncertainty, the irreducible randomness in
future outcomes. We use it only to mark predictions that warrant caution in the
actor update.

\paragraph{Ensemble statistics.}
\bcppo{} trains $M$ independently initialized, bootstrapped scalar cost critics
$\{Q_c^{(m)}(s,a)\}_{m=1}^{M}$ that take $(s,a)$ as input and estimate expected
discounted future cost. They share the same scalar TD-target definition, but
independent initialization and per-member binary bootstrap masks expose them to
different selections of training samples: each mask randomly keeps or drops a sample from one
member's loss. Members share neither parameters nor gradients, and no extra loss
forces their predictions apart;
\suppref{Appendix}{app:implementation} gives exact masking and architectures.
For a given $(s,a)$, the ensemble mean and variance are:
\begin{align}
  \mu_c(s,a) &= \frac{1}{M}\sum_{m=1}^{M}Q_c^{(m)}(s,a), \label{eq:mu}\\
  \sigma_c^2(s,a) &= \frac{1}{M}\sum_{m=1}^{M}
    \!\left(Q_c^{(m)}(s,a)-\mu_c(s,a)\right)^{\!2}. \label{eq:sigma}
\end{align}
Each member produces one expected-cost prediction, not a return quantile or an
environment sample. Thus $\sigma_c$ measures disagreement across trained critics at
fixed $(s,a)$, not outcome variance. The cost advantage $A_c$ and PID handle
predicted and observed mean cost; only disagreement enters the actor penalty.
To prevent numerical collapse when all critics do converge, we enforce
$\sigma_c\leftarrow\max(\sigma_c,\sigma_{\min})$ with $\sigma_{\min}=10^{-6}$.

\paragraph{Bachelier expected-excess penalty.}
At each fixed $(s,a)$, we approximate the discrete critic outputs by
$Q_c(s,a)\sim\mathcal{N}(\mu_c,\sigma_c^2)$. We call this a \emph{local Gaussian
surrogate}; it approximates variation across trained critics, not randomness in
environment outcomes. We then define a state--action-dependent reference level,
called the \emph{strike} in the Bachelier formula:
\begin{equation}
  K(s,a)=\mu_c(s,a)+c_0\sigma_c(s,a),
  \qquad c_0=\Phi^{-1}(\alpha)+\kappa,
  \label{eq:strike}
\end{equation}
where $\Phi^{-1}(\alpha)$ is the standard-normal $\alpha$-quantile and
$\kappa\geq 0$ shifts the reference farther into the upper tail. Thus $c_0$
sets where excess begins within the Gaussian approximation.
The Bachelier-inspired shaping penalty is the expected excess beyond this strike:
\begin{equation}
  \begin{aligned}
  \RB(s,a)
  &= \mathrm{Call}(K;\mu_c,\sigma_c)\\
  &= (\mu_c-K)\Phi(d_\kappa)+\sigma_c\,\phi(d_\kappa),\\
  d_\kappa&=-c_0=-\Phi^{-1}(\alpha)-\kappa.
  \end{aligned}
  \label{eq:bachelier_penalty}
\end{equation}
Substituting Eq.~(\ref{eq:strike}) gives the implemented one-dimensional form:
\begin{equation}
  \RB(s,a)
  = \sigma_c(s,a)
    \bigl[\phi(-c_0)-c_0\Phi(-c_0)\bigr],
  \label{eq:bachelier_simplified}
\end{equation}
The mean cancels because $A_c$ already supplies predicted cost, avoiding a
second count of the same signal.
For fixed $(\alpha,\kappa)$, $h(c_0)=\phi(-c_0)-c_0\Phi(-c_0)$ is a positive
constant and
\begin{equation}
  \beta_{\mathrm{eff}}=\beta h(c_0),\qquad
  \beta\RB=\beta_{\mathrm{eff}}\sigma_c,
  \label{eq:effective_beta}
\end{equation}
so the implemented signal is an analytically scaled disagreement penalty.
Settings with the same $\beta_{\mathrm{eff}}$ produce the same actor
coefficient, so this update cannot distinguish separate effects of
$\alpha$, $\kappa$, and $\beta$.
Bachelier supplies expected-excess and local Gaussian CVaR-excess
interpretations, not another uncertainty statistic or a safety guarantee.

Before entering the actor update, $\RB$ is detached from the computation graph:
$\stopgrad{\RB}(s,a)=\mathrm{stopgrad}(\RB(s,a))$.
Thus gradients from the actor loss do not update ensemble-critic parameters or
alter their TD targets.

\subsection{Actor-Critic Learning and Mean-Cost Control}

\bcppo{} retains PPO's state-value networks $V_r$ and $V_c$, which predict
future reward and cost from a state and are used for Generalized Advantage
Estimation (GAE). The action-conditioned critics $\{Q_c^{(m)}\}$ instead provide
disagreement for a specific state--action pair.

\paragraph{Standard TD cost-critic learning.}
All ensemble critics are trained with standard TD targets, independent of the
Bachelier-inspired shaping penalty:
\begin{align}
  y_t &= c_t + \gamma(1-d_t)\,\widehat{V}_c(s_{t+1}), \label{eq:td_target}\\
  L_{\mathrm{cost}}^{(m)} &=
    \frac{\sum_{t\in B}z_{t}^{(m)}
    \left(Q_c^{(m)}(s_t,a_t)-y_t\right)^2}
    {\max\{1,\sum_{t\in B}z_{t}^{(m)}\}}. \label{eq:critic_loss}
\end{align}
Here $B$ is the minibatch and $z_t^{(m)}\sim\mathrm{Bernoulli}(0.8)$ is
independently resampled for every member and sample at each minibatch update.
The flag $d_t$ marks termination, and $\widehat{V}_c$ is the detached rollout-time
cost-value prediction; no separate target network is maintained. All critics
share $y_t$ but have independent parameters, initializations, and masks.

\paragraph{Branch-normalized hybrid advantage.}
Let $A_r$ and $A_c$ denote the reward and cost advantages---estimates of how
much better an action is than its state baseline---computed via
GAE~\citep{schulman2018highdimensionalcontinuouscontrolusing}.
\bcppo{} constructs a branch-normalized hybrid advantage:
\begin{equation}
  \ABP = \norm{A_r} - \lambda\cdot\norm{A_c + \beta\,\stopgrad{\RB}},
  \label{eq:hybrid_advantage}
\end{equation}
where $\norm{\cdot}$ denotes batch zero-mean unit-variance normalization,
$\lambda\geq 0$ is the controller-set weight on the cost branch, and $\beta>0$
weights disagreement within that branch.
The branches are normalized independently so large reward magnitudes do not
suppress the disagreement signal.
Because $\lambda$ multiplies the complete cost-plus-disagreement branch,
disagreement affects the actor only when the mean-cost multiplier is positive.
When $\lambda=0$, the actor update reduces to reward PPO; when the mean-cost
controller sets $\lambda>0$, disagreement changes which sampled actions receive
a stronger safety penalty. The spread term is therefore not an independent tail
constraint.
The actor is updated with the PPO clipped objective~\citep{schulman2017proximalpolicyoptimizationalgorithms}:
\begin{equation}
  L_{\mathrm{actor}} = -\E\!\left[\min\!\left(\rho_t\ABP,\;
    \mathrm{clip}(\rho_t,1-\epsilon,1+\epsilon)\ABP\right)\right],
  \label{eq:ppo_loss}
\end{equation}
where $\rho_t=\pi_\theta(a_t\mid s_t)/\pi_{\theta_{\mathrm{old}}}(a_t\mid s_t)$.
The scalar $\epsilon>0$ is the PPO clipping radius.

\paragraph{Anti-windup PID Lagrangian controller.}
The mean-cost weight $\lambda$ is updated after every batch of environment
transitions collected by the current policy. Let $B_k$ denote this rollout batch
at cycle $k$ and
$e_k=\mathrm{mean}(c_t\in B_k)-\Clim$ its mean-cost error.
The controller first forms a candidate integral and projected update:
\begin{equation}
  \begin{aligned}
  \widetilde I_k &= I_{k-1}+e_k,\\
  \widetilde u_k &= \lambda_k+K_Pe_k+K_I\widetilde I_k
    +K_D(e_k-e_{k-1}),\\
  \lambda_{k+1} &= \Pi_{[0,\lambda_{\max}]}(\widetilde u_k).
  \end{aligned}
  \label{eq:pid}
\end{equation}
Here $K_P$, $K_I$, and $K_D$ are the proportional, integral, and derivative
gains, and $\Pi_{[0,\lambda_{\max}]}$ clips its argument to the displayed
interval.
Integral wind-up occurs when $\lambda$ reaches a boundary but its integral term
keeps accumulating.
\bcppo{} applies \emph{anti-windup}: if $\widetilde u_k$ lies beyond a projection
boundary and $e_k$ points farther outward, the stored state remains
$I_k=I_{k-1}$; otherwise the candidate integral is accepted. If the multiplier
is saturated and the error reverses direction, the stored integral is clipped
to help it return to the allowed range. We set $\lambda_{\max}=50$.

Algorithm~\ref{alg:bcppo} summarizes the full update; implementation choices
and default hyperparameters are reported with the experimental setup.

\begin{algorithm}[t]
\caption{\bcppo{} Update Step}
\label{alg:bcppo}
\begin{algorithmic}[1]
\STATE \textbf{Input:} Policy $\pi_\theta$; value networks $V_r,V_c$;
  critics $\{Q_c^{(m)}\}_{m=1}^{M}$; PID state; multiplier $\lambda$
\STATE Collect rollout batch $\mathcal{B}=\{(s_t,a_t,r_t,c_t,s_{t+1},d_t)\}$
  under $\pi_\theta$
\STATE Compute $A_r,A_c$ via GAE; compute $\hat{V}_c(s_{t+1})$ for TD targets
\STATE Update multiplier: $e_k\!\leftarrow\!\mathrm{mean}(c_t)\!-\!\Clim$;
  apply Eq.~(\ref{eq:pid}) with anti-windup logic
\FOR{each PPO epoch (one pass through the rollout data) and minibatch}
  \STATE Compute $\mu_c,\sigma_c$ from the current critics via
    Eqs.~(\ref{eq:mu})--(\ref{eq:sigma}); clamp
    $\sigma_c\leftarrow\max(\sigma_c,10^{-6})$
  \STATE Compute $K$ and $\RB$ via
    Eqs.~(\ref{eq:strike})--(\ref{eq:bachelier_penalty}); detach $\RB$
  \STATE Construct $\ABP$ and minimize $L_{\mathrm{actor}}$ via
    Eqs.~(\ref{eq:hybrid_advantage})--(\ref{eq:ppo_loss})
  \STATE Compute TD target $y_t$ via Eq.~(\ref{eq:td_target})
  \FOR{each critic $m=1,\ldots,M$}
    \STATE Minimize its bootstrap-masked TD loss
      $L_{\mathrm{cost}}^{(m)}$ via Eq.~(\ref{eq:critic_loss})
  \ENDFOR
\ENDFOR
\end{algorithmic}
\end{algorithm}

\subsection{Computational Complexity and Theoretical Takeaways}
\label{sec:complexity}

\paragraph{Computational cost.}
\bcppo{} uses the same actor-only forward pass as vanilla PPO at deployment;
the ensemble is discarded after training. For batch size $B$, its training
the overhead is $\mathcal O(BM)$ scalar-critic evaluation, TD learning, and
aggregation, with $\mathcal O(M)$ critic storage. Since
$\RB=\sigma_c h(c_0)$, the Bachelier map itself is $\mathcal O(B)$ and its
scalar coefficient can be precomputed. This retains ordinary gradient-based PPO updates:
unlike Monte Carlo CVaR it needs no trajectory sorting, unlike QR PPO it needs
no $N_q$-output quantile loss, and unlike CPO it needs no second-order
constraint solve. Under our defaults it uses
$M=5$ scalar outputs instead of QR PPO's $N_q=32$ quantile outputs per
transition. These operation counts do not claim that \bcppo{} is always faster
in elapsed time; full operation counts and measured inference latency are
reported together in \suppref{Appendix}{app:complexity}.

\paragraph{Theoretical takeaways.}
Full statements and proofs appear in \suppref{Appendix}{app:proofs}. They establish:
\begin{enumerate}
  \item \textbf{Smooth explicit sensitivity.}
  The state--action-dependent reference level reduces the penalty to
  $\RB=\sigma_c h(c_0)$: sensitivity to $\mu_c$ vanishes and sensitivity to
  $\sigma_c$ is finite
  (\suppref{Proposition}{prop:smoothness}).
  \item \textbf{A Gaussian CVaR-excess interpretation.}
  For $\kappa=0$, $\RB=(1-\alpha)(\CVaR_\alpha-\VaR_\alpha)$ under the local
  Gaussian approximation to critic outputs; $\kappa>0$ shifts the excess
  threshold. This is not an exact
  trajectory-CVaR constraint (\suppref{Proposition}{prop:cvar_connection}).
  \item \textbf{Scale separation.}
  Branch normalization is invariant to independent positive rescaling of its
  two branches (\suppref{Remark}{prop:scale_equiv}). Auxiliary results rely on
  stated approximations or idealized assumptions and do not establish
  numerical agreement with true error or risk, convergence, or safety.
\end{enumerate}

\section{Experiments}
\label{sec:experiments}

\subsection{Setup}

\paragraph{Implementation summary.}
Unless stated otherwise, \bcppo{} uses $M=5$ critics,
$(\alpha,\beta,\kappa)=(0.95,0.15,0.15)$, PPO clipping
$\epsilon=0.2$, and $\lambda_{\max}=50$. Algorithm~\ref{alg:bcppo} specifies
the update, while complete architectures, optimizer settings, and
hyperparameters appear in \suppref{Appendix}{app:implementation}.

\paragraph{Environments and protocol.}
We evaluate on five continuous-control tasks spanning Safety-Gymnasium
navigation (Goal1, Button1, Push1) and MuJoCo locomotion (Hopper-v4, Ant-v4).
Every method is trained for $1\mathrm{M}$ steps with five independent random
seeds and evaluated
under the same rare-event cost wrapper and, within each environment, the same
cost limit.
These shared tasks, costs, budgets, interaction counts, and evaluation seeds
form the common comparison protocol; details are in
\suppref{Appendix}{app:implementation}.

\paragraph{Rare-event cost wrapper.}
Rather than relying on built-in Safety-Gymnasium event costs, we use the
controlled cost
\begin{equation}
\label{eq:cost_wrapper}
c(s,a)=c_a\|a\|_2^2+c_b\,\mathbf{1}\{\|o(s)\|_2>15\},
\end{equation}
where $c_a,c_b>0$ are environment-specific coefficients and $o(s)$ denotes the
observation extracted from state $s$.
The indicator adds an occasional bounded cost spike to the small action cost, so
a policy can satisfy a mean-cost budget while still exhibiting a large upper
tail. We use the common
threshold as a controlled stress test, not as a physically identical boundary
definition across the five observation spaces.
Original environment costs are logged as diagnostics but are not optimized.

\paragraph{Baselines.}
Table~\ref{tab:main_core} compares maintained mean-cost baselines (PPOLag, CPO,
and FOCOPS), our PPO implementations of CPPO and QR PPO, and tail-aware WCSAC
\citep{Yang_Simao_Tindemans_Spaan_2021}. CPPO uses Monte Carlo upper-tail
weighting; QR PPO uses an action-conditioned quantile cost critic. For OmniSafe,
we convert both the constraint statistic and cost-advantage branch from
episodic-cost units to the per-step units used by our wrapper, while leaving
each maintained algorithm update unchanged. WCSAC is our PyTorch reproduction,
checked against its published Gaussian equations and a fixed version of the
public code; it is not an author-supplied run. All methods follow the common
comparison protocol.

\paragraph{Metrics.}
\textbf{Return} measures task performance (higher is better).
\textbf{Cost rate} is episode-average step cost (lower is better).
\textbf{Safety Rate} is the fraction satisfying the budget,
$\Pr(\mathrm{Cost\ rate}\leq\Clim)$, while \textbf{CVaR@95} is the mean cost
rate in the worst $5\%$ tail. \textbf{Worst Gap}
$=\widehat{\mathrm{CVaR}}_{0.95}-\Clim$ measures tail severity relative to the
budget. The main table reports Return, Cost rate, and CVaR@95; supplementary
tables retain all metrics. Each policy uses 20 fixed evaluation seeds
($10000$--$10019$) and
deterministic actions, so per-run CVaR@95 is the largest observed cost rate;
tables aggregate over five trained policies.

\subsection{Main Results}

\newcommand{\pmv}[2]{#1\,$\pm$\,#2}
\newcommand{\pmvb}[2]{\textbf{#1}\,$\boldsymbol{\pm}$\,\textbf{#2}}
\begin{table*}[t]
\centering
\begingroup
\small
\setlength{\tabcolsep}{2.2pt}
\renewcommand{\arraystretch}{1.08}
\begin{tabular}{@{}llccccc@{}}
\toprule
\textbf{Method} & \textbf{Metric}
& \textbf{Goal1} & \textbf{Button1} & \textbf{Push1}
& \textbf{Hopper} & \textbf{Ant} \\
\midrule
\multirow{3}{*}{\textbf{\bcppo{}}} & Return & \pmvb{17.37}{6.48} & \pmvb{15.28}{5.43} & \pmvb{0.72}{0.53} & \pmvb{772.53}{229.81} & \pmvb{1286.97}{320.40} \\
& Cost & \pmv{0.123}{0.254} & \pmv{0.152}{0.083} & \pmv{0.098}{0.129} & \pmv{0.012}{0.002} & \pmv{0.117}{0.192} \\
& CVaR@95 & \pmvb{0.173}{0.361} & \pmvb{0.254}{0.125} & \pmvb{0.198}{0.203} & \pmvb{0.013}{0.002} & \pmvb{0.159}{0.236} \\
\addlinespace[1.5pt]
\multirow{3}{*}{CPPO} & Return & \pmvb{17.10}{6.36} & \pmv{14.16}{4.57} & \pmv{0.19}{0.94} & \pmvb{607.33}{133.97} & \pmvb{910.18}{235.03} \\
& Cost & \pmv{0.021}{0.021} & \pmv{0.177}{0.085} & \pmv{0.128}{0.121} & \pmv{0.008}{0.001} & \pmv{0.025}{0.018} \\
& CVaR@95 & \pmvb{0.053}{0.074} & \pmv{0.275}{0.093} & \pmv{0.284}{0.145} & \pmvb{0.009}{0.001} & \pmvb{0.067}{0.067} \\
\addlinespace[1.5pt]
\multirow{3}{*}{QR PPO} & Return & \pmv{9.96}{9.46} & \pmvb{18.60}{2.07} & \pmv{0.32}{0.40} & \pmv{533.26}{197.22} & \pmv{797.36}{134.50} \\
& Cost & \pmv{0.156}{0.220} & \pmv{0.273}{0.174} & \pmv{0.143}{0.190} & \pmv{0.010}{0.001} & \pmv{0.050}{0.088} \\
& CVaR@95 & \pmv{0.349}{0.554} & \pmvb{0.416}{0.257} & \pmv{0.282}{0.336} & \pmv{0.011}{0.001} & \pmv{0.123}{0.238} \\
\addlinespace[1.5pt]
\multirow{3}{*}{PPOLag} & Return & \pmv{18.75}{7.04} & \pmv{14.75}{4.69} & \pmv{0.41}{0.33} & \pmvb{3498.48}{118.00} & \pmvb{3943.41}{647.48} \\
& Cost & \pmv{1.461}{0.337} & \pmv{0.997}{0.246} & \pmv{0.435}{0.301} & \pmv{0.089}{0.063} & \pmv{0.480}{0.258} \\
& CVaR@95 & \pmv{2.013}{0.276} & \pmv{1.353}{0.283} & \pmv{1.332}{0.852} & \pmvb{0.099}{0.074} & \pmvb{0.540}{0.269} \\
\addlinespace[1.5pt]
\multirow{3}{*}{CPO} & Return & \pmvb{25.09}{1.88} & \pmvb{20.38}{3.59} & \pmv{0.31}{0.31} & \pmvb{2341.07}{1353.40} & \pmvb{4573.81}{594.98} \\
& Cost & \pmv{1.117}{0.865} & \pmv{1.049}{0.236} & \pmv{0.691}{0.428} & \pmv{0.013}{0.002} & \pmv{0.652}{0.242} \\
& CVaR@95 & \pmvb{1.695}{1.714} & \pmvb{1.455}{0.367} & \pmv{1.464}{0.747} & \pmvb{0.015}{0.004} & \pmvb{0.749}{0.229} \\
\addlinespace[1.5pt]
\multirow{3}{*}{FOCOPS} & Return & \pmv{16.33}{7.45} & \pmv{10.01}{5.50} & \pmv{\textminus{}0.29}{0.52} & \pmv{3222.89}{651.22} & \pmv{3792.30}{461.57} \\
& Cost & \pmv{2.358}{0.687} & \pmv{1.609}{0.307} & \pmv{2.290}{1.344} & \pmv{0.138}{0.059} & \pmv{0.596}{0.261} \\
& CVaR@95 & \pmv{3.463}{0.874} & \pmv{2.457}{1.119} & \pmv{6.203}{3.413} & \pmv{0.180}{0.068} & \pmv{0.674}{0.271} \\
\addlinespace[1.5pt]
\multirow{3}{*}{WCSAC$^{\dagger}$} & Return & \pmvb{\textminus{}0.01}{0.16} & \pmv{\textminus{}4.80}{11.15} & \pmv{\textminus{}2.25}{3.03} & \pmvb{3128.71}{666.13} & \pmv{\textminus{}194.82}{69.22} \\
& Cost & \pmv{0.006}{0.001} & \pmv{1.256}{1.964} & \pmv{0.608}{0.628} & \pmv{0.010}{0.002} & \pmv{3.906}{0.320} \\
& CVaR@95 & \pmvb{0.006}{0.001} & \pmv{1.384}{2.138} & \pmv{1.078}{0.770} & \pmvb{0.016}{0.009} & \pmv{4.443}{0.244} \\
 
\bottomrule
\end{tabular}
\endgroup
\caption{
	  Shared-protocol final results over 175 runs (mean$\pm$standard deviation
  across five training seeds after $1\mathrm{M}$ interactions; 20 fixed
  evaluation seeds).
  Higher Return is better, whereas lower Cost rate and CVaR@95 are better.
  Bold Return/CVaR entries are not jointly worse than another displayed mean.
  $^\dagger$WCSAC is our checked Gaussian reproduction, not an author-supplied
  run. Values are regenerated from raw episodes.
}
\label{tab:main_core}
\end{table*}

Table~\ref{tab:main_core} reports the 175 runs; Figure~\ref{fig:pareto_frontier}
shows reward--CVaR positions, and \suppref{Appendix}{app:training_curves} shows
training. Because Eq.~(\ref{eq:cost_wrapper}) can produce rare boundary costs
despite an acceptable mean, we report the trade-off rather than one scalar
winner.

\paragraph{Reward--CVaR pairs and trade-offs.}
Here one method empirically dominates another if its five-seed mean has no lower
Return and no higher CVaR, with at least one strict inequality.
At the five-seed mean level, \bcppo{} is non-dominated in all five tasks: it
empirically dominates 13 of 30 comparator--task means and trades reward against
CVaR in the other 17. It dominates all six comparators on Push1, four on
Button1, two on Goal1, and WCSAC on Ant. Its navigation CVaR is below PPOLag,
CPO, and FOCOPS, whereas Hopper and Ant baselines often exchange
higher tail cost for higher return. WCSAC's results characterize only our
wrapper-specific reproduction. Push1 provides \bcppo{}'s clearest joint
advantage.

\begin{figure*}[t]
\centering
\includegraphics[width=\textwidth]{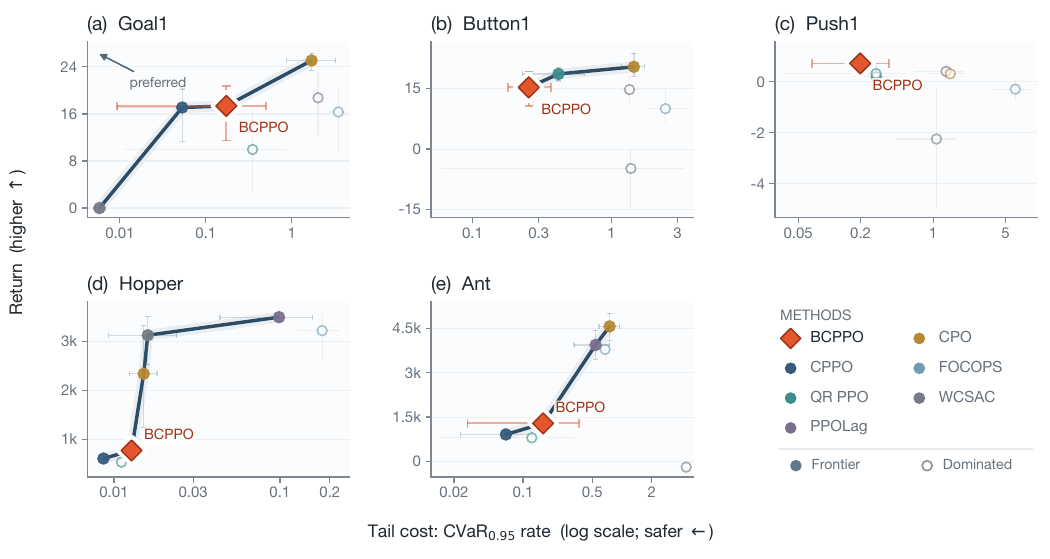}
\caption{
  Five-seed mean reward--CVaR trade-offs; bars are 95\% bootstrap CIs across
  training runs. Upper-left is preferred; filled markers and dark segments show
  the finite-set empirical frontier. Orange diamonds mark \bcppo{}; WCSAC is our
  checked Gaussian reproduction. Log axes preserve dominance.
}
\label{fig:pareto_frontier}
\end{figure*}

\paragraph{Statistical scope.}
This frontier is descriptive: only Push1--FOCOPS, Push1--WCSAC, and
Ant--WCSAC have 95\% intervals, formed by resampling the five training seeds,
that favor \bcppo{} on both axes. These intervals do not correct for simultaneous
reward--CVaR testing or multiple comparisons. The supported
claim is therefore a non-dominated five-seed mean pair in each task---not
maximal reward, minimal CVaR, or universal superiority.

\paragraph{Mean cost versus the observed tail.}
Mean cost and CVaR move together in some cells but not enough to make one
redundant. At the run level, 18 of 175 final policies satisfy the mean-cost
budget while their empirical CVaR exceeds it, directly exhibiting the motivating
gap under this wrapper. Conversely, most method--task cells do not show that
pattern, and the table does not imply that every expectation-constrained method
must have worse tail behavior. We therefore report Return, mean Cost, and CVaR
separately.

\paragraph{Training and deployment cost.}
\bcppo{} trains five cost critics but deploys only the PPO actor.
\Suppref{Appendix}{app:inference} reports noisy, overlapping policy latencies;
they do not establish a speed advantage over CPPO or QR PPO.

\subsection{Ablation Summary}

Five-seed, 500k-step ablations on Button1, Push1, and Hopper test component
removal, uncertainty placement, the effective coefficient, and ensemble size;
full results appear in \suppref{Appendix}{app:ablations}.

\paragraph{Spread penalty and placement.}
Removing the spread penalty changes the reward--CVaR trade-off rather than
uniformly worsening risk: the full method favors return on Button1, lower CVaR
on Push1, and improves both means on Hopper with large variation across seeds.
Placement is
more consistent. Turning disagreement into a reward bonus raises navigation
CVaR, most clearly on Push1; Hopper is nearly unchanged. The penalty is therefore
an environment-dependent conservative bias, not a universal tail reducer.

\paragraph{Normalization and mean-cost control.}
Removing branch normalization improves Button1 but produces the largest Push1
instability (CVaR@95 $0.934\pm1.479$) and degrades Hopper; its role is scale
control, not universal CVaR reduction. Anti-windup is more consistent: disabling
it raises mean CVaR in all three environments, most on Button1 ($0.414$ to
$0.668$). Overlapping five-seed variation prevents a significance claim.

\paragraph{Coefficient and ensemble sensitivity.}
The $\beta$, $\alpha$, and $\kappa$ sweeps all change
$\beta_{\mathrm{eff}}=\beta h(\Phi^{-1}(\alpha)+\kappa)$, not three independent
mechanisms. Non-monotone Push1 results make the default a single shared
coefficient choice, not an environment-specific optimum. With $M=1$ spread stays
at its numerical floor and
CVaR variance is largest; $M=10$ gives no clear gain over $M=5$. A
direct $\sigma_c$ penalty with its coefficient set to
$\beta_{\mathrm{eff}}$ is algebraically identical. Keeping the same numerical
$\beta$ instead makes the direct penalty $69.2\times$ stronger and tests scale,
not additional risk information.

\paragraph{Does disagreement respond to missing coverage?}
For 15 fixed policies, we retrain ensembles after structured data removal and
equal-size controls. We use area under the receiver-operating-characteristic
curve (AUC) to measure how well disagreement identifies the deliberately omitted
region on separate evaluation episodes. The AUC gains are $0.168$
and $0.099$ (95\% CIs $[0.118,0.218]$ and $[0.044,0.153]$). Against matched
one-standard-deviation target-noise controls, the gains remain $0.126$ and
$0.110$. Missing coverage therefore raises disagreement in this test, but does
not make it a general error score or tail-event probability; details appear in
\suppref{Appendix}{app:sigma_calibration}.

\paragraph{Does the penalty steer actor updates?}
In a paired one-update test, we clone each actor; treatment enables the
disagreement penalty and control disables it, while rollout, split, optimizer
initialization, and minibatch order are matched.
For five saved models
with $\lambda>0$, larger-$\RB$ actions receive lower treatment-minus-control log
probability after removing the rank association with predicted mean cost
(Spearman rank correlation, which compares monotone rank orderings,
$-0.119$, 95\% CI $[-0.216,-0.039]$; all five
negative).
Four saved models with $\lambda=0$ and 72 both-disabled checks have exactly zero
parameter difference. An ordinary linear-regression check is inconclusive, so
this establishes the direction of one update, not lower final-policy CVaR;
\suppref{Appendix}{app:actor_intervention} gives the protocol.

\section{Discussion}
\label{sec:discussion}

\subsection{Justification}

\bcppo{} separates predicted cost, handled by the cost advantage and multiplier,
from detached sensitivity to critic training; Bachelier supplies its
expected-excess coefficient. The method has a non-dominated mean reward--CVaR
pair, and controlled retraining identifies omitted regions more readily than
equal-size controls. Mixed ablations and a multiplier positive in only 0.1\% of
logged Ant updates prevent attributing every frontier result to spread alone.
The method targets sparse, costly, weakly covered violations when training can
afford multiple critics but deployment requires one policy.

\subsection{Limitations}

The Gaussian approximation can miss multimodal or heavy-tailed structure.
Agreement can reflect shared bias, and disagreement can reflect benign
instability. Its value is therefore a cautionary training signal, not an error
probability, formal safety guarantee, or control for well-covered aleatoric
hazards.
Because the multiplier gates the full cost--disagreement branch, $\lambda=0$
makes the actor update reward-only, so \bcppo{} provides no independent tail-risk
pressure while the mean-cost constraint is inactive.

The wrapper is a controlled stress test, not a physical hazard model. With 20
evaluation episodes, CVaR@95 is the largest observed cost, and five seeds give
coarse uncertainty estimates. CPPO and QR PPO span two optimizer settings,
WCSAC is our reproduction, and shared hyperparameters are not task-optimal.
OmniSafe cost-rate units are aligned, but controllers were not jointly retuned
or budget-swept, so violations describe these configurations. The paired
intervention tests one post-training update with a newly initialized optimizer,
not downstream CVaR.

\subsection{Future Work}

Future work should combine disagreement with distributional cost critics and a
separate constraint active when the mean-cost multiplier is zero. Controlled
continuations could test downstream CVaR; broader evaluation should add physical
hazards, wrapper sweeps, seeds, and episodes. Shared features or distillation
could reduce cost, while analysis should connect coverage to realized tails.

\section{Conclusion}
\label{sec:conclusion}

\bcppo{} combines critic-disagreement shaping with anti-windup mean-cost
control. Across five tasks, no tested comparator has both higher mean return and
lower mean CVaR; deliberate data-removal controls also show increased
disagreement in omitted regions. The evidence supports a practical
reward--caution balance with policy-only deployment, not universal tail-risk
reduction or a safety guarantee.

\bibliography{references}

\ifdefined\WITHAPPENDIX
\clearpage
\appendix
\raggedbottom
\setcounter{topnumber}{4}
\setcounter{bottomnumber}{2}
\setcounter{totalnumber}{6}
\setcounter{dbltopnumber}{4}
\renewcommand{\topfraction}{0.95}
\renewcommand{\bottomfraction}{0.90}
\renewcommand{\textfraction}{0.04}
  \renewcommand{\floatpagefraction}{0.55}
  \renewcommand{\dbltopfraction}{0.95}
  \renewcommand{\dblfloatpagefraction}{0.55}
  \setlength{\textfloatsep}{9pt plus 2pt minus 2pt}
  \setlength{\floatsep}{9pt plus 2pt minus 2pt}
  \setlength{\dbltextfloatsep}{8pt plus 2pt minus 2pt}
  \setlength{\dblfloatsep}{6pt plus 1pt minus 1pt}
  \setlength{\intextsep}{9pt plus 2pt minus 2pt}
  \makeatletter
  \setlength{\@fptop}{0pt}
  \setlength{\@fpsep}{5pt plus 1pt}
  \setlength{\@fpbot}{0pt plus 1fil}
  \setlength{\@dblfptop}{0pt}
  \setlength{\@dblfpsep}{5pt plus 1pt}
  \setlength{\@dblfpbot}{0pt plus 1fil}
\makeatother
\section*{Appendix}

\section{Theoretical Properties and Proofs}
\label{app:theory}

\subsection{Proofs of Propositions}
\label{app:proofs}

This appendix gives the formal statements and proofs of the three analytical
properties referenced in the main text.

\begin{proposition}[Explicit Smoothness and Bounded Local Sensitivity]
\label{prop:smoothness}
Let $\mathrm{Call}(K;\mu,\sigma)$ denote the Bachelier expected-excess term
in Eq.~(\ref{eq:bachelier}) with $\sigma>0$.
For a fixed strike $K$, its partial derivatives are
\[
 \frac{\partial \mathrm{Call}}{\partial \mu}=\Phi(d)\in[0,1],
 \qquad
 \frac{\partial \mathrm{Call}}{\partial \sigma}=\phi(d)\leq (2\pi)^{-1/2}.
\]
For the moving reference strike used by \bcppo{}, $K=\mu+c_0\sigma$ with
$c_0=\Phi^{-1}(\alpha)+\kappa$, the penalty reduces to
\[
 \RB(\mu,\sigma)=\sigma\,h(c_0),\qquad
 h(c_0)=\phi(-c_0)-c_0\Phi(-c_0),
\]
so $\partial \RB/\partial\mu=0$ and
$\partial \RB/\partial\sigma=h(c_0)$, a finite constant for fixed
$(\alpha,\kappa)$.
\end{proposition}

\begin{proof}
For fixed $K$, differentiating Eq.~(\ref{eq:bachelier}) gives
$\partial \mathrm{Call}/\partial\mu=\Phi(d)$ and
$\partial \mathrm{Call}/\partial\sigma=\phi(d)$ after cancellation of the
terms containing $\partial d$. For the implemented strike
$K=\mu+c_0\sigma$, $d=-c_0$ and
$\RB=\sigma[\phi(-c_0)-c_0\Phi(-c_0)]$, so the sensitivity to $\mu$ is zero
and the sensitivity to $\sigma$ is the fixed scalar $h(c_0)$.
\end{proof}

\begin{proposition}[Connection to CVaR Excess Under Gaussian Surrogate]
\label{prop:cvar_connection}
Under the local Gaussian surrogate $Q_c(s,a)\sim\mathcal{N}(\mu_c,\sigma_c^2)$,
when $\kappa=0$ (i.e., the strike equals $\VaR_\alpha$):
\begin{equation}
  \begin{aligned}
  \RB\big|_{\kappa=0}
  &= \E[(Q_c-\VaR_\alpha)_+]\\
  &= (1-\alpha)\bigl(\CVaR_\alpha(Q_c)-\VaR_\alpha(Q_c)\bigr).
  \end{aligned}
 \label{eq:cvar_relation}
\end{equation}
For $\kappa>0$, the strike is shifted above $\VaR_\alpha$, so the
expected excess is smaller than the VaR-strike excess while measuring exceedance
beyond a higher reference threshold. This algebra alone does not imply that the
learned policy becomes more conservative as $\kappa$ increases.
\end{proposition}

\begin{proof}
For $X\sim\mathcal{N}(\mu,\sigma^2)$:
\begin{align*}
  \E[(X-\VaR_\alpha)_+]
  &= \int_{\VaR_\alpha}^\infty (x-\VaR_\alpha)f(x)\,dx \\
  &= \int_{\VaR_\alpha}^\infty x\,f(x)\,dx\\
  &\quad - \VaR_\alpha
    \int_{\VaR_\alpha}^\infty f(x)\,dx \\
  &= (1-\alpha)\CVaR_\alpha(X)-\VaR_\alpha(1-\alpha) \\
  &= (1-\alpha)\CVaR_\alpha(X)\\
  &\quad -(1-\alpha)\VaR_\alpha(X).
\end{align*}
Monotonicity in $\kappa$ follows because $K$ increases with $\kappa$, while
\[
  \frac{\partial \mathrm{Call}}{\partial K}=-\Phi(d)\leq 0 .
\]
Thus the expected excess decreases as the reference threshold increases.
\end{proof}

\begin{remark}[Scale-Equivariance of Branch-Normalized Advantage]
\label{prop:scale_equiv}
Batch zero-mean unit-variance normalization satisfies
$\norm{cX}=\norm{X}$ for any $c>0$ and non-constant vector $X$.
Hence, for
$\ABP=\norm{A_r}-\lambda\cdot\norm{A_c+\beta\stopgrad{\RB}}$
and any $c_1,c_2>0$,
\begin{align*}
  \norm{c_1A_r} &= \norm{A_r},\\
  \norm{c_2(A_c+\beta\stopgrad{\RB})} &= \norm{A_c+\beta\stopgrad{\RB}}.
\end{align*}
Thus $\ABP$ is invariant to independent rescaling of the reward
and cost-risk branches, and the Lagrange weight $\lambda$ retains a
dimensionless interpretation as a relative priority between branches.
\end{remark}

\subsection{Bachelier Formula: Derivation and Financial Interpretation}
\label{app:bachelier_derivation}

The Bachelier model~\citep{bachelier1900} was the first formal mathematical
model for option pricing, assuming arithmetic (additive) Brownian motion for
the underlying price process, as opposed to the later geometric Brownian
motion assumption of Black-Scholes.
The expected payoff of a European call option under arithmetic Brownian motion
with $X_T\sim\mathcal{N}(\mu,\sigma^2)$ is exactly Eq.~(\ref{eq:bachelier}).
In BCPPO we repurpose this formula for a different domain: rather than
modeling option payoffs, we model the expected excess of an uncertain cost
estimate over a moving reference threshold.
The ``strike price'' $K$ in Eq.~(\ref{eq:strike}) is a moving reference
threshold, and the Bachelier call value quantifies expected excess under the
critic-output surrogate. Because this threshold is tied to $(\mu_c,\sigma_c)$, the resulting
value is exactly $h(c_0)\sigma_c$, not a separate estimate of trajectory-tail
probability.

\subsection{Extended Theoretical Diagnostics}
\label{app:extended_hardcore_theory}

This section provides idealized analytical diagnostics for selected components
of \bcppo{}. These statements are deliberately narrower than policy-safety or
end-to-end convergence guarantees. We maintain notation:
$\RB(s,a)$ for the Bachelier-inspired shaping penalty, $\mathrm{norm}(\cdot)$ for batch
normalization, and $c_0=\Phi^{-1}(\alpha)+\kappa$ for the moving reference-strike
coefficient. Table~\ref{tab:theory_scope} summarizes the scope, limitations, and
ways to check the results below.

\subsubsection{Distribution-Free Critic-Output Bound}
\label{app:tail_theory}

The moving reference strike $K=\mu_c+c_0\sigma_c$ used in Eq.~(\ref{eq:strike}) has
an interpretation beyond the Gaussian surrogate through the following
distribution-free bound.

\begin{remark}[Distribution-Free Critic-Output Bound]
\label{prop:tail_bound}
At a fixed $(s,a)$, let $J$ be uniform on $\{1,\ldots,M\}$ and define the
discrete surrogate variable $X=Q_c^{(J)}(s,a)$. Its mean and standard deviation
are exactly the implemented $\mu_c(s,a)$ and $\sigma_c(s,a)$. For
$\sigma_c(s,a)>0$, Cantelli's one-sided Chebyshev inequality gives, for any
$c>0$,
\begin{equation}
  \Pr_J(X \ge \mu_c + c\sigma_c) \le \frac{1}{1+c^2}.
\end{equation}
Consequently, the BCPPO moving reference strike
$K(s,a)=\mu_c(s,a)+c_0\sigma_c(s,a)$ with
$c_0=\Phi^{-1}(\alpha)+\kappa>0$ satisfies
\begin{equation}
  \Pr_J\!\bigl(Q_c^{(J)}(s,a)\ge K(s,a)\bigr) \le \frac{1}{1+c_0^2}.
\end{equation}
Under the separate Gaussian surrogate the corresponding exact probability is
$\Pr(Q_c\ge K)=\Phi(-c_0)$, which is generally much smaller than the Cantelli
upper bound rather than making that bound tight. Both statements concern the
surrogate distribution over critic outputs, not the realized trajectory-cost
distribution, and therefore do not guarantee any upper bound on the probability
of violating an environmental cost limit.
\end{remark}

\begin{table*}[t]
\centering
\begingroup
\small
\setlength{\tabcolsep}{2.8pt}
\renewcommand{\arraystretch}{1.08}
\begin{tabularx}{\textwidth}{@{}>{\raggedright\arraybackslash}p{0.17\textwidth}*{3}{>{\raggedright\arraybackslash}X}@{}}
\toprule
\textbf{Result} & \textbf{Meaning} & \textbf{Main limitation} &
\textbf{Can it be tested, and how?} \\
\midrule
Critic-output bound
(Remark~\ref{prop:tail_bound}) &
Bounds how often an ensemble member exceeds the moving strike without assuming
Gaussian outputs. &
The bound can be loose and concerns critic outputs, not environment trajectory
costs. &
Yes: compare the observed member-exceedance rate with the Cantelli and Gaussian
values across saved model states and controlled data-removal experiments. \\
\cmidrule(lr){1-4}
Bias--variance diagnostic
(Proposition~\ref{prop:bachelier_variance}) &
Shows that the idealized spread score avoids the explicit
$(1-\alpha)^{-1}$ variance factor of a tail-indicator score. &
Uses a known population quantile and unclipped trajectory scores; the surrogate
bias is unbounded. &
Partly: estimate gradient variance and deviation from a high-sample CVaR
reference while varying $N$ and $\alpha$. \\
\cmidrule(lr){1-4}
Strike and coefficient coupling
(Lemma~\ref{prop:kappa_monotonicity}) &
Increasing $\kappa$ raises the strike but lowers the raw penalty; only
$\beta_{\mathrm{eff}}$ controls its scale. &
An algebraic statement at fixed spread; it does not predict the policy reached
after learning. &
Yes: run coefficient-matched $(\alpha,\kappa,\beta)$ settings and compare update
signals and learning curves. \\
\cmidrule(lr){1-4}
Critic-error propagation
(Proposition~\ref{prop:lipschitz_error_propagation}) &
The moving-strike penalty is insensitive to common mean error and linear in
spread error. &
Assumes reference mean and spread from an idealized critic-output distribution,
plus uniform estimation-error bounds; it does not show that spread numerically
matches true error. &
Yes: perturb held-out $\mu_c$ and $\sigma_c$ separately and compare the measured
$\Delta\RB$ with the analytical sensitivity. \\
\cmidrule(lr){1-4}
Anti-windup accumulation
(Proposition~\ref{prop:pid_stability}) &
The stored integral cannot grow while the multiplier is outward-saturated. &
Does not establish joint actor--controller stability or convergence. &
Yes: log saturation intervals and compare integral state and multiplier traces
with and without anti-windup. \\
\cmidrule(lr){1-4}
Finite-ensemble concentration
(Proposition~\ref{prop:finite_sample_concentration}) &
Characterizes how idealized variance-estimation error contracts with ensemble
size $M$. &
Requires bounded independent and identically distributed outputs, whereas
learned critics remain correlated. &
Partly: subsample larger ensembles, vary bootstrap diversity, and measure spread
error versus $M$ and inter-critic correlation. \\
\bottomrule
\end{tabularx}
\endgroup
\caption{
  Meaning, limitations, and empirical checks for the extended
  theoretical results. ``Test'' denotes a diagnostic of the stated result,
  not validation of an end-to-end safety guarantee.
}
\label{tab:theory_scope}
\end{table*}

\subsubsection{Idealized Tail-Score Bias--Variance Diagnostic}
\label{app:bias_variance}

\begin{proposition}[Idealized Trajectory-Score Bias--Variance Diagnostic]
\label{prop:bachelier_variance}
Let $p_\theta(\tau)$ be the trajectory density and
$w(\tau)=\nabla_\theta\log p_\theta(\tau)
=\sum_t\nabla_\theta\log\pi_\theta(a_t\mid s_t)$ satisfy $\|w(\tau)\|\le W$.
Let independent trajectory costs be bounded $C(\tau)\in[0,C_{\max}]$.
Assume a continuous cost distribution and let
$z_\alpha=\VaR_\alpha(C)$ be the population quantile. For the idealized score
estimator
\[
  \widehat g_{\rm CVaR}
  =\frac{1}{N}\sum_{i=1}^N
    \frac{w_i(C_i-z_\alpha)_+}{1-\alpha},
\]
standard score-function regularity gives
$\E[\widehat g_{\rm CVaR}]=g_*=\nabla_\theta\CVaR_\alpha(C)$, and
\begin{equation}
  \mathbb E[\|\widehat g_{\rm CVaR}-g_*\|^2]
  \le \frac{W^2 C_{\max}^2}{N(1-\alpha)}.
\end{equation}
Let $S(\tau)\geq0$ be a detached trajectory-level aggregation of ensemble spread
with finite second moment, and assume
$c_0=\Phi^{-1}(\alpha)+\kappa\geq0$. The corresponding spread-surrogate estimator
\[
  \widehat g_B
  =\frac{h(c_0)}{N}\sum_{i=1}^N w(\tau_i)\,S(\tau_i)
\]
has mean-squared error
\begin{align}
  \mathbb E[\|\widehat g_B-g_*\|^2]
  &= B^2 + \mathrm{Var}(\widehat g_B) \\
  &\le B^2 + \frac{h(c_0)^2 W^2}{N}\,\mathbb E[S(\tau)^2].
\end{align}
Here $\mathrm{Var}(Y)=\mathbb E[\|Y-\mathbb E Y\|^2]$ for a vector-valued random
variable $Y$, and $B=\|\mathbb E[\widehat g_B]-g_*\|$ is the generally
unquantified bias
from replacing trajectory-level CVaR with critic spread. Since
$|h(c_0)|\le (2\pi)^{-1/2}$, the displayed variance contribution is bounded for
fixed $(\alpha,\kappa)$ and carries no $(1-\alpha)^{-1}$ factor.
\end{proposition}

\begin{proof}
Because $\|w(\tau_i)\|\le W$, $(C_i-z_\alpha)_+\le C_{\max}$, and
$\Pr(C_i\ge z_\alpha)=1-\alpha$, the per-sample second moment is at most
$W^2C_{\max}^2/(1-\alpha)$. Independence, unbiasedness, and
$\mathrm{Var}(Y)\le\E\|Y\|^2$ give the stated $1/N$ bound.

For $\widehat g_B$, the standard bias--variance decomposition yields
$\mathbb E[\|\widehat g_B-g_*\|^2]=\|\mathbb E[\widehat g_B]-g_*\|^2+
\mathrm{Var}(\widehat g_B)$. The variance term is bounded as before by
$h(c_0)^2 W^2\mathbb E[S(\tau)^2]/N$; the bias $B$ is the fixed
surrogate mismatch between the Bachelier gradient and the true CVaR gradient.
\end{proof}

\begin{remark}
This comparison uses the population quantile and an unclipped score estimator.
The implemented actor instead applies transition-level shaping inside a clipped,
branch-normalized PPO update; it is not the trajectory estimator displayed above.
The implemented CPPO comparison instead uses an empirical quantile and clipped
weights and is generally biased; the proposition does not analyze it or the
full clipped, branch-normalized PPO update. The absence of an explicit
$(1-\alpha)^{-1}$ term for $\widehat g_B$ must also be weighed against $B$,
which can dominate and is not bounded here.
\end{remark}

\subsubsection{Strike Monotonicity and Parameter Coupling}
\label{app:parameter_coupling}

\begin{lemma}[Strike Monotonicity and Parameter Coupling]
\label{prop:kappa_monotonicity}
Under the local Gaussian surrogate
$Q_c(s,a)\sim\mathcal N(\mu_c,\sigma_c^2)$ with $\sigma_c>0$, the moving reference strike
$K(\kappa)=\mu_c+c_0(\kappa)\sigma_c$ with
$c_0(\kappa)=\Phi^{-1}(\alpha)+\kappa$ satisfies
$\partial K/\partial\kappa=\sigma_c>0$. The exceedance probability
$p(\kappa)=\Pr(Q_c\ge K(\kappa))=\Phi(-c_0(\kappa))$,
and the raw penalty $\RB=\sigma_ch(c_0)$ satisfy
\begin{equation}
  \frac{\partial p}{\partial\kappa}=-\phi(c_0)<0,
  \qquad
  \frac{\partial \RB}{\partial \kappa}
  =-\sigma_c\Phi(-c_0)<0.
\end{equation}
Moreover, the actor-side shaping term depends on $(\alpha,\kappa,\beta)$ only
through
$\beta_{\rm eff}=\beta h(\Phi^{-1}(\alpha)+\kappa)$. Hence these parameters are
coupled, and increasing $\kappa$ at fixed $\beta$ weakens the raw disagreement
shaping coefficient.
\end{lemma}

\begin{proof}
Since $\partial c_0/\partial\kappa=1$, differentiating the strike gives
$\partial K/\partial\kappa=\sigma_c$. Differentiating
$p(\kappa)=\Phi(-c_0)$ gives $\partial p/\partial\kappa=-\phi(c_0)$.
Finally, $h'(c_0)=-\Phi(-c_0)$, so
$\partial\RB/\partial\kappa=\sigma_ch'(c_0)$; substituting the definition of
$h$ into $\beta\RB$ gives the stated $\beta_{\rm eff}$.
\end{proof}

The falling exceedance probability is partly definitional because the event
threshold moves; neither it nor the smaller raw penalty implies that the learned
policy becomes safer as $\kappa$ increases.

\subsubsection{Critic-Error Propagation}
\label{app:error_propagation}

\begin{assumption}[Bounded Reference and Learned Critic Statistics]
\label{ass:critic_bounds}
Let $(\mu_c^*,\sigma_c^*)$ denote reference statistics of an idealized
critic-output distribution, not ground-truth environment risk, and let
$(\hat{\mu}_c,\hat{\sigma}_c)$ be the learned finite-ensemble statistics. Assume
both are uniformly bounded over the compact state-action space
$\mathcal{S}\times\mathcal{A}$.
That is,
$\max(\|\mu_c^*\|_\infty,\|\hat{\mu}_c\|_\infty)\leq Q_{\max}$ and
$\sigma_{\min}\leq \min(\sigma_c^*,\hat{\sigma}_c)\leq\sigma_{\max}$.
\end{assumption}

\begin{proposition}[Bounded Error Propagation from Critic Approximation]
\label{prop:lipschitz_error_propagation}
Under Assumption~\ref{ass:critic_bounds}, let critic errors be bounded by
$\|\hat{\mu}_c-\mu_c^*\|_\infty\leq\epsilon_\mu$ and
$\|\hat{\sigma}_c-\sigma_c^*\|_\infty\leq\epsilon_\sigma$. Under the moving reference strike
$K=\mu_c+c_0\sigma_c$, the propagation of critic error into the Bachelier
penalty satisfies
\begin{equation}
  \begin{aligned}
  &\left|\RB(\hat{\mu}_c,\hat{\sigma}_c)
  -\RB(\mu_c^*,\sigma_c^*)\right| \\
  &\qquad \leq L_\mu\epsilon_\mu+L_\sigma\epsilon_\sigma,
  \end{aligned}
\end{equation}
where $L_\mu=0$ and
$L_\sigma=h(c_0)\triangleq\phi(-c_0)-c_0\Phi(-c_0)$.
\end{proposition}

\begin{proof}
Substituting $K=\mu_c+c_0\sigma_c$ into the Bachelier formula yields
$\RB(\mu_c,\sigma_c)=\sigma_c h(c_0)$, where $h(c_0)$ is a fixed positive
scalar for constant $(\alpha,\kappa)$. The total differential with respect to
$(\mu_c,\sigma_c)$ gives
\begin{align*}
  \frac{\partial \RB}{\partial \mu_c} &= 0,\\
  \frac{\partial \RB}{\partial \sigma_c} &= h(c_0).
\end{align*}
Applying the Mean Value Theorem on the convex domain of critic outputs,
\begin{align*}
  &\left|\RB(\hat{\mu}_c,\hat{\sigma}_c)
    -\RB(\mu_c^*,\sigma_c^*)\right| \\
  &\quad \leq
    \left|\frac{\partial \RB}{\partial \mu_c}\right|\epsilon_\mu
    +\left|\frac{\partial \RB}{\partial \sigma_c}\right|\epsilon_\sigma \\
  &\quad = h(c_0)\epsilon_\sigma .
\end{align*}
Setting $L_\mu=0$ and $L_\sigma=h(c_0)$ completes the proof. Since $c_0\geq0$,
$h(c_0)\leq(2\pi)^{-1/2}$.
\end{proof}

\begin{remark}[Structural Decoupling]
Proposition~\ref{prop:lipschitz_error_propagation} formalizes the intended
separation: this branch is insensitive to a common shift in critic means and
depends only on disagreement. This sensitivity statement does not show that
disagreement numerically tracks true error or that actor optimization is stable.
\end{remark}

\subsubsection{Bounded Integral Accumulation under Anti-Windup}
\label{app:anti_windup_theory}

\begin{proposition}[No Integral Accumulation during Outward Saturation]
\label{prop:pid_stability}
Let $\widetilde I_k=I_{k-1}+e_k$ and $\widetilde u_k$ be the candidate quantities
in Eq.~(\ref{eq:pid}). Suppose that, throughout a consecutive interval,
$\widetilde u_k>\lambda_{\max}$ with $e_k>0$, or $\widetilde u_k<0$ with
$e_k<0$. The implemented anti-windup rule retains $I_k=I_{k-1}$ at every such
step, while projection ensures $\lambda_{k+1}\in[0,\lambda_{\max}]$. Hence the
integral contribution cannot grow during that outward-saturated interval.
\end{proposition}

\begin{proof}
Under either stated condition, the anti-windup branch rejects
$\widetilde I_k$ and retains $I_k=I_{k-1}$. Induction therefore makes the stored
integral constant over the interval. The projected update places every
$\lambda_{k+1}$ in $[0,\lambda_{\max}]$.
\end{proof}

\begin{remark}[Scope]
This result establishes the narrow property implemented by the anti-windup
guard. It does not prove local or global convergence of the PID multiplier,
actor, or non-convex CMDP optimization.
\end{remark}

\subsubsection{Finite-Ensemble Concentration}
\label{app:ensemble_concentration}

\begin{assumption}[Idealized Independent Ensemble Outputs]
\label{ass:bootstrap_bounds}
At a fixed $(s,a)$, the random variables $\{Q_c^{(m)}(s,a)\}_{m=1}^M$ are
independent and identically distributed (i.i.d.) and bounded in
$[-Q_{\max},Q_{\max}]$. This is an analytical
idealization: the implemented critics share rollouts and TD targets, so
independent initialization and bootstrap masks do not make their outputs
exactly independent.
\end{assumption}

\begin{proposition}[Finite-Sample Concentration of $\sigma_c^2$]
\label{prop:finite_sample_concentration}
Under Assumption~\ref{ass:bootstrap_bounds}, let $\mu_c=\E[Q_c]$ and define the
empirical variance
$(\sigma_c^{(M)})^2=\frac{1}{M}\sum_{m=1}^M(Q_c^{(m)}-\hat\mu_c)^2$
with $\hat\mu_c=\frac{1}{M}\sum_{m=1}^M Q_c^{(m)}$.
Let $(\sigma_c^{\mathrm{true}})^2=\mathrm{Var}(Q_c)$ denote the population
variance under this idealized critic-output distribution.
Write $V=\mathrm{Var}((Q_c-\mu_c)^2)$ and $L=\log(4/\delta)$.
For any $\delta\in(0,1)$,
\begin{equation}
  \Pr\!\left(
    \left|(\sigma_c^{(M)})^2-(\sigma_c^{\mathrm{true}})^2\right|
    \geq
    \sqrt{\frac{2VL}{M}} + \frac{10Q_{\max}^2L}{3M}
  \right)
  \leq \delta .
\end{equation}
Consequently, whenever $\sigma_c^{\mathrm{true}}>0$,
\begin{equation}
\begin{split}
  \Pr\!\Bigg(
    &\left|\sigma_c^{(M)}-\sigma_c^{\mathrm{true}}\right| \\
    &\geq
    \frac{1}{\sigma_c^{\mathrm{true}}}
    \left(
      \sqrt{\frac{2VL}{M}} + \frac{10Q_{\max}^2L}{3M}
    \right)
  \Bigg)
  \leq \delta .
\end{split}
\end{equation}
\end{proposition}

\begin{proof}
Write $\sigma^2=(\sigma_c^{\mathrm{true}})^2$, $\hat\mu=\hat\mu_c$, and
$\hat\sigma^2=(\sigma_c^{(M)})^2$.
Expanding the centered empirical variance gives
\[
  \hat\sigma^2
  = \frac{1}{M}\sum_{m=1}^M (Q_c^{(m)}-\mu_c)^2 - (\hat\mu-\mu_c)^2 .
\]
Let $Z_m=(Q_c^{(m)}-\mu_c)^2$. Then $Z_m\in[0,4Q_{\max}^2]$,
$\E[Z_m]=\sigma^2$, and $\mathrm{Var}(Z_m)=V$.
Bernstein's inequality yields, with probability at least $1-\delta/2$,
\[
  \left|\frac{1}{M}\sum_{m=1}^M Z_m - \sigma^2\right|
  \leq
  \sqrt{\frac{2VL}{M}} + \frac{4Q_{\max}^2L}{3M} .
\]
Hoeffding's inequality yields, with probability at least $1-\delta/2$,
\[
  |\hat\mu-\mu_c|
  \leq Q_{\max}\sqrt{\frac{2L}{M}},
\]
so $(\hat\mu-\mu_c)^2\leq \frac{2Q_{\max}^2L}{M}$.
A union bound over the two events gives the first display.
The second display follows from
$|\hat\sigma-\sigma|=|\hat\sigma^2-\sigma^2|/(\hat\sigma+\sigma)
\leq |\hat\sigma^2-\sigma^2|/\sigma$.
\end{proof}

\begin{remark}
Under the i.i.d. assumption, this controls only finite-ensemble estimation noise
around the variance of the idealized critic-output distribution. It does not
validate it as a Bayesian distribution over model uncertainty,
establish OOD detection by itself, or apply unchanged to correlated learned
critics. The empirical
diagnostics in Appendix~\ref{app:sigma_calibration} provide controlled evidence
that missing coverage raises spread while showing that its numerical values do
not generally track prediction error or tail-event probability.
\end{remark}

\setcounter{secnumdepth}{2}
\section{Implementation Details and Auxiliary Diagnostics}
\label{app:empirical}

\subsection{Implementation Details}
\label{app:implementation}

\paragraph{Baseline implementation details.}
The main table includes our PPO implementations of CPPO and QR PPO, three
maintained OmniSafe implementations, and our checked reproduction of Gaussian
WCSAC. CPPO computes the total cost
$C_{\mathrm{rollout}}$ of each rollout, estimates its empirical
$\alpha$-quantile $\eta=q_\alpha(C_{\mathrm{rollout}})$, and forms a Monte Carlo
(MC) cost advantage by
\begin{equation}
  A_c^{\mathrm{MC}}
  = A_c\,
    \mathrm{clip}\!\left(
      \frac{\mathbf{1}\{C_{\mathrm{rollout}}\geq\eta\}}{1-\alpha},
      0,w_{\max}
    \right).
\end{equation}
where $w_{\max}$ is the maximum multiplier applied after clipping. QR PPO is the
distributional-critic diagnostic rather than a maintained safe-RL baseline: it
trains an action-conditioned quantile cost critic with $N_q=32$ quantile heads
and uses the mean of quantiles with $\tau\geq\alpha$ as a detached upper-tail
signal added to the cost branch. The PPOLag, CPO, and FOCOPS rows use OmniSafe
0.5.0 under the same rare-event cost wrapper, cost limit,
total-step budget, and seed count. Because OmniSafe's maintained objectives are
written in episodic-cost units, we divide the episodic-cost statistic by the
observed mean episode length and rescale the cost-advantage branch consistently.
We still log the raw episodic cost. This conversion makes the optimizer's units
match the per-step rate used by every other row. The WCSAC row follows the
public Gaussian WCSAC mean/variance safety critics and CVaR-weighted SAC actor
objective using the fixed public-code revision \texttt{eaeff3e7}, with
the per-step budget converted to its
equivalent scale for a discounted cost-return critic. It is therefore labeled a
reproduction, not an author-supplied official run. Its near-zero Goal1 return
and extreme Ant failure should be read only as outcomes of this reproduction
under the synthetic wrapper. The retained runs cannot isolate whether Gaussian
cost modeling, optimization that reuses stored transitions, or transferring
published hyperparameters to our wrapper is
responsible, so those outcomes are not generalized to WCSAC outside this
protocol. Per-run configurations, protocol identifiers, raw episodes, and training
curves accompany the released table data.

\paragraph{Ensemble critics.}
All $M=5$ critics share the same multilayer-perceptron (MLP) architecture (two
hidden layers of 256 units with Tanh activations) but have independently drawn
initial parameters. They share the minibatch order and ordinary scalar TD target,
which contains no disagreement penalty. Each member instead receives an
independent binary mask that keeps each TD-loss entry with probability $0.8$;
this perturbs the data seen by each member while
preserving the target
$y_t=c_t+\gamma(1-d_t)V_c(s_{t+1})$.
No gradients or parameters are shared, and no auxiliary loss explicitly pushes
the predictions apart. Thus $\sigma_c$ reflects disagreement among separately
trained cost predictors rather than an imposed diversity objective.

\paragraph{Numerical stability.}
$\sigma_c$ is clamped to $\sigma_{\min}=10^{-6}$.
The ratio $d_\kappa=(\mu_c-K)/\sigma_c$ is clipped to $[-10,10]$ before
passing to $\Phi$ and $\phi$ to prevent very small computer-represented values
from rounding to zero.

\paragraph{PID controller.}
Anti-windup is implemented by freezing the integral term when the projected
mean-cost multiplier is saturated and the current error would drive it deeper into
saturation; when the error reverses, the integral term is allowed to unwind.
The default PID gains used in the experiments are $K_P=0.2$, $K_I=0.02$, and
$K_D=0.05$.

\paragraph{Cost-limit selection.}
All $\Clim$ values are specified in per-step cost units and applied
directly to the batch-mean cost $\mathrm{mean}(c_t\in B_k)$ in the PID
error signal, avoiding episode-length dependence. The final protocol uses
$\Clim=0.65$ for Goal1, Button1, and Push1, $0.015898692$ for Hopper, and
$3.4$ for Ant. These are environment-level cost-rate limits shared unchanged by
all methods, not method-specific budgets. The three navigation values were
inherited from the original production protocol; the Hopper value was set in an
early 100k-step preliminary PPO run to $0.85$ times that run's seed-0 evaluation
cost; and the
Ant value was fixed in the archived Ant experiment grid before the
shared-protocol OmniSafe and WCSAC comparisons. These records explain when the
budgets were chosen; they are not a budget-selection study planned in advance on
held-out runs. We
therefore condition all claims on the listed limits and do not claim robustness
to budget selection.

\paragraph{Hyperparameter selection record.}
The main BCPPO settings $(\alpha,\beta,\kappa,M)=(0.95,0.15,0.15,5)$ and PID
gains were fixed across environments before the final shared-protocol baseline
runs. The subsequent 500k sensitivity grid was analyzed only as a diagnostic and was not
used to replace main-table policies. Development included earlier engineering
preliminary engineering runs and no held-out tuning plan registered before the
experiments; accordingly, the paper treats
these values as shared defaults, not as uniquely optimal settings. In
particular, the sensitivity table contains alternative Push1 mean points with
both higher return and lower CVaR than the default at 500k steps.

\paragraph{Cost wrapper.}
For all main experiments, the wrapper computes the optimization cost using
Eq.~(\ref{eq:cost_wrapper}). Goal1 and Push1 use $c_a=0.005$ and $c_b=10$;
Button1, Hopper, and Ant use $c_a=0.01$ and $c_b=5$. The observation threshold
is fixed to $15$ for all tasks. Because observation coordinates and dimensions
differ across task families, this threshold defines a common synthetic stress
test rather than the same physical boundary event in every environment. The
same cost function, cost limit, and evaluation metrics are applied to every
method in a given environment. If an environment implementation exposes
a built-in safety cost, it is logged for reference but is not used by the
policy update or the reported main cost metrics. Appendix or auxiliary runs
under built-in Safety-Gymnasium costs are therefore reported separately and are
not mixed with the controlled rare-event main table.

\paragraph{Reproducibility details.}
Table~\ref{tab:repro_hardware} summarizes the hardware and software snapshot
used for the final main-table and ablation runs, and
Table~\ref{tab:repro_hyperparams} records the main training protocol and key
BCPPO hyperparameters. The inference benchmark protocol is described in
Appendix~\ref{app:inference}. CPU denotes central processing unit, RAM denotes
random-access memory, and GiB/TiB denote gibibytes/tebibytes. In the tables
below, ``internal'' denotes the
\bcppo{}, CPPO, and QR PPO implementations from our codebase.

\paragraph{Released-material scope.}
The experiments do not consume a fixed dataset or require an offline
preprocessing stage: every method collects trajectories online from the named
simulators. The accompanying release contains the training and analysis source,
wrapper definitions, launch grids, retained per-run configurations, final
episode records, and aggregation scripts. Table~\ref{tab:repro_hyperparams} and
the optimizer-setting paragraph below explicitly identify the limited set of
individual per-run configuration files that are no longer retained; we do not
reconstruct those missing files.

\begin{table*}[t]
\centering
\begingroup
\small
\setlength{\tabcolsep}{4pt}
\renewcommand{\arraystretch}{1.05}
\begin{tabularx}{\textwidth}{@{}>{\raggedright\arraybackslash}p{0.23\textwidth}>{\raggedright\arraybackslash}X@{}}
\toprule
\textbf{Execution group} & \textbf{Specification} \\
\midrule
Internal accelerator server & 16$\times$ PPU-ZW810E; 180 logical CPUs; 1.7~TiB RAM; Ubuntu 24.04.2; Python 3.12.3; PyTorch 2.6.0 \\
\cmidrule(lr){1-2}
Internal A800/evaluation & 4$\times$ NVIDIA A800-SXM4-80GB; 128 logical CPUs; 900~GiB RAM; Ubuntu 22.04.5; Python 3.10.20; PyTorch 2.7.1 \\
\cmidrule(lr){1-2}
OmniSafe baseline runs & 2$\times$ NVIDIA A800-SXM4-80GB; 128 logical CPUs; 400~GiB RAM; Ubuntu 22.04.5; Python 3.10.20; PyTorch 2.7.1; OmniSafe 0.5.0 \\
\cmidrule(lr){1-2}
WCSAC reproduction & 2$\times$ PPU-ZW810E; 40 logical CPUs; 200~GiB RAM; Ubuntu 24.04.2; Python 3.12.3; PyTorch 2.9.0 \\
\cmidrule(lr){1-2}
Shared environment stack & NumPy 1.26.0; Gymnasium 0.28.1; Safety-Gymnasium 1.0.0; MuJoCo 2.3.3 \\
\cmidrule(lr){1-2}
Thread limits & \texttt{OMP\_NUM\_THREADS}, \texttt{MKL\_NUM\_THREADS}, \texttt{OPENBLAS\_NUM\_THREADS}, \texttt{NUMEXPR\_NUM\_THREADS}, and Torch threads set to 1 for high-concurrency runs \\
\bottomrule
\end{tabularx}
\endgroup
\caption{
  Hardware and software environment.
  Main-table runs span the disclosed execution groups below, with one
  accelerator device assigned to each training run. Hardware affects elapsed
  training time and
  throughput (environment interactions per second); the environment, interaction budget, seeds, wrapper, and final
  evaluation protocol define the common comparison.
}
\label{tab:repro_hardware}
\end{table*}

\begin{table*}[t]
\centering
\begingroup
\small
\setlength{\tabcolsep}{4pt}
\renewcommand{\arraystretch}{1.05}
\begin{tabularx}{\textwidth}{@{}>{\raggedright\arraybackslash}p{0.23\textwidth}>{\raggedright\arraybackslash}X@{}}
\toprule
\textbf{Setting} & \textbf{Value} \\
\midrule
Main-table scale & 5 environments $\times$ 7 methods $\times$ 5 seeds = 175 runs \\
\cmidrule(lr){1-2}
Environments & SafetyPointGoal1/\-Button1/\-Push1-v0, Hopper-v4, Ant-v4 \\
\cmidrule(lr){1-2}
Training horizon & Nominally $10^6$ interactions; internal PPO 1{,}000{,}192--1{,}000{,}448 (rollout boundary), OmniSafe 999{,}424 (epoch boundary), WCSAC exactly 1{,}000{,}000 \\
\cmidrule(lr){1-2}
Final evaluation & 20 episodes per run; fixed seeds 10000--10019 \\
\cmidrule(lr){1-2}
Internal rollout steps / minibatch / epochs & Early 25 Button1/Push1 rows: 256 / 64 / 4; remaining 50 rows: 1024 / 512 / 3 \\
\cmidrule(lr){1-2}
OmniSafe steps per epoch / minibatch / update iterations & 2048 / 256 / 20 (CPO: 10 policy and 10 iterative-solver steps) \\
\cmidrule(lr){1-2}
WCSAC minibatch / stored-transition capacity / update ratio & 256 / $10^6$ / one update per environment step \\
\cmidrule(lr){1-2}
Internal PPO learning rate / hidden size / $\gamma$ / GAE $\lambda$ & $3\mathrm{e}{-4}$ / 256 / 0.99 / 0.95 \\
\cmidrule(lr){1-2}
WCSAC learning rate / hidden size / $\gamma$ & $10^{-3}$ / 2$\times$256 / 0.99 \\
\cmidrule(lr){1-2}
BCPPO ensemble size $M$ & 5 \\
\cmidrule(lr){1-2}
BCPPO $(\alpha,\beta,\kappa)$ & $(0.95,0.15,0.15)$ in every environment \\
\cmidrule(lr){1-2}
BCPPO PID gains $(K_P,K_I,K_D)$ & $(0.2,0.02,0.05)$ \\
\cmidrule(lr){1-2}
BCPPO $\lambda_{\max}$ & 50 in every environment \\
\cmidrule(lr){1-2}
BCPPO bootstrap keep probability & 0.8 \\
\bottomrule
\end{tabularx}
\endgroup
\caption{
  Main training protocol and BCPPO hyperparameters.
  Exact per-run configurations are retained for all 75 OmniSafe runs, all 25
  WCSAC runs, and 50 of 75 internal runs. The remaining 25 early internal runs
  retain their full training logs and settings recorded by the versioned launch
  script, but their individual \texttt{train\_config.json} files are no longer
  available. We report the launch-script settings explicitly rather than
  reconstruct per-run files.
}
\label{tab:repro_hyperparams}
\end{table*}

\paragraph{PPO optimizer-setting record.}
The 25 early Button1/Push1 internal rows used the versioned launch-script defaults
(rollout 256, minibatch 64, four update epochs), whereas the 50 runs with retained
per-run configurations used the later setting (1024, 512, three). BCPPO uses one setting
within each environment, but the Button1/Push1 CPPO and QR PPO seed sets contain
a mixture of the two settings. All rows retain the same interaction budget,
wrapper, training seeds, and fixed final-evaluation protocol. We release the
exact 50 retained configuration files and document the other 25 through the
versioned launch-script settings. This difference in optimizer settings is a
limitation of our PPO comparison; no result value is missing or
reconstructed.

\paragraph{Relation to contemporaneous tail methods.}
EVO fits extreme samples with extreme-value objectives
\citep{gao2025evo}; SL-SAC combines implicit-quantile cost critics with a
CVaR-constrained Langevin SAC update~\citep{keswani2026slsac}; and SteinGate
uses a Stein-discrepancy certificate to switch into
recovery~\citep{chemingui2026steingate}. These methods target outcome-tail
modeling or distributional certification. \bcppo{} instead uses
scalar-critic disagreement as a training-only reliability signal and does not
claim the guarantees or outcome-distribution coverage of those approaches.

\subsection{Finite-Set Pareto Definition}
\label{app:pareto}

Figure~\ref{fig:pareto_frontier} treats the seven trained methods as a finite
candidate set. Within each environment, a method is non-dominated if no other
candidate has weakly lower empirical CVaR@95 and weakly higher return, with one
strict inequality. Solid segments order the observed non-dominated means by
CVaR; they are visual guides, not interpolated policies. The logarithmic CVaR
axes in Figure~\ref{fig:pareto_frontier} are monotone transformations and thus
do not change this finite-set membership. Frontier membership is descriptive:
overlapping intervals are not interpreted as statistically significant
superiority.

\subsection{Computational Complexity Details}
\label{app:complexity}

\begin{table*}[t]
\centering
\begingroup
\small
\setlength{\tabcolsep}{5pt}
\renewcommand{\arraystretch}{1.05}
\begin{tabularx}{\textwidth}{@{}>{\raggedright\arraybackslash}p{0.14\textwidth}>{\raggedright\arraybackslash}p{0.28\textwidth}>{\raggedright\arraybackslash}X@{}}
\toprule
\textbf{Method} & \textbf{Forward cost} & \textbf{Additional operations} \\
\midrule
Vanilla PPO & $1\times$ actor + $1\times$ critic & GAE only \\
\bcppo{} & $+M$ cost-critic passes & $\mathcal O(BM)$ mean/variance + $\mathcal O(B)$ Bachelier call \\
CPPO (Monte Carlo CVaR) & $1\times$ actor + $1\times$ critic & sort $N_{\rm traj}$ trajectory costs (indicator weighting per rollout) \\
QR PPO  & $+1$ quantile critic with $N_q$ heads & $\mathcal O(BN_q)$ quantile regression \\
CPO     & $1\times$ actor + $1\times$ critic & repeated products in an iterative second-order constraint solver \\
\bottomrule
\end{tabularx}
\endgroup
\caption{
  Additional per-update complexity relative to vanilla PPO.
  $B$ is batch size, $M$ is ensemble size, $N_q$ is the number of quantile
  heads, and $N_{\rm traj}$ is the number of trajectories.
}
\label{tab:complexity}
\end{table*}

Table~\ref{tab:complexity} summarizes the dominant additional costs relative
to vanilla PPO. QR PPO scales linearly with $N_q$; CPPO sorts rollout costs;
CPO uses an iterative second-order constraint solver; and \bcppo{} evaluates
$M$ cost critics. These are different computational profiles: \bcppo{} avoids
trajectory sorting and second-order products but pays for the ensemble during training, so
the operation-count comparison alone does not establish an elapsed-time
advantage.

\newpage
\subsection{Multiplier Activity and Ensemble Diversity Audit}
\label{app:lambda_ensemble_audit}

Across the 25 final-main runs, $\lambda$ is positive for 14.2\% of logged actor
updates when each run receives equal weight, with substantial task variation
(Goal1: 18.4\%; Button1: 7.6\%; Push1: 17.0\%; Hopper: 27.9\%; Ant:
0.1\%). Thus disagreement cannot explain Ant's final operating point. This
directly confirms that the implemented mechanism is conditional on the
mean-cost controller rather than an always-active tail constraint.

The 21-checkpoint/probe audit combines 15 retained final-main checkpoints for
Goal1, Hopper, and Ant with six explicitly labeled 100k diagnostic checkpoints
for Button1 and Push1, whose legacy final checkpoints were unavailable. The
saved critics learn strongly correlated common cost functions, with mean
pairwise correlation 0.917. After removing that common prediction at every
probe, however, member deviations have normalized effective rank 0.578, so the
ensembles are not identical copies. Full run-level values and confidence
intervals accompany the released analysis. This functional-diversity check
does not establish statistical independence, calibrated epistemic
uncertainty, or safety.

\subsection{Inference Latency Benchmark}
\label{app:inference}

Table~\ref{tab:inference_latency} reports per-forward inference latency on a
CPU for the final policies from the main 1M-step runs. Timings use batch size 1,
100 untimed warmup iterations, and 1000 timed iterations. Policy-only inference is the
deployment path; policy-plus-ensemble timing is a conservative diagnostic upper
bound that includes the $M=5$ cost-critic ensemble forward pass used to compute
$\mu_c$ and $\sigma_c$. The CPPO and QR PPO saved model states (checkpoints)
created by our codebase
store an unused ensemble module for architectural compatibility, but those
algorithms never evaluate it. Their policy-plus-ensemble entries therefore time
a hypothetical extra computation, not their actual inference paths.

\begin{table*}[t]
\centering
\begingroup
\small
\setlength{\tabcolsep}{4pt}
\renewcommand{\arraystretch}{1.06}
\begin{tabular}{@{}lcc@{}}
\toprule
\textbf{Method} & \textbf{Policy (ms)} & \textbf{With ensemble (ms)} \\
\midrule
BCPPO & $0.66 \pm 0.66$ & 3.58 \\
CPPO & $0.89 \pm 0.54$ & 4.81 \\
QR PPO & $0.89 \pm 0.47$ & 5.02 \\
 
\bottomrule
\end{tabular}
\endgroup
\caption{
  Per-forward inference latency on CPU at batch size 1.
  ``ms'' denotes milliseconds; policy entries report mean $\pm$ standard
  deviation across successful saved-model timings. We omit averaged reciprocal
  throughput because it is not the reciprocal of mean latency.
}
\label{tab:inference_latency}
\end{table*}

All three policy-only timings have large overlapping variation, so the measured
differences do not establish that one actor is faster. The relevant architectural
fact is that BCPPO discards the training-time ensemble at deployment and executes
one PPO actor forward pass.

\subsection{Transition-Level Critic-Disagreement Diagnostic}
\label{app:sigma_calibration}

We reevaluate the 15 available 1M-step BCPPO saved model states from Ant, Hopper,
and Goal1 with the same 20 deterministic evaluation seeds (10000--10019) used by
Table~\ref{tab:main_core}. The early Button1 and Push1 runs no longer have saved
model states, so they are excluded rather than reconstructed. The resulting 300
episodes contain
223,607 unique transitions. At every transition we retain the ensemble mean and
spread after multiplying both discounted critic outputs by $(1-\gamma)$ to put
them on the same per-step scale as the reported cost rate. We also retain the
exact wrapper cost and boundary indicator, which is one when
$\|o(s)\|_2>15$ and zero otherwise. We compute a Monte Carlo
cost-to-go target $\sum_{j=t}^{T-1}\gamma^{j-t}c_j$ from each completed episode
and scale it by the same factor. We additionally evaluate a
same-state comparison action obtained by adding Gaussian noise with standard
deviation $0.25$ times the action range and clipping to the valid action bounds;
this action is diagnostic only and is never executed.

For each trained policy, we compute: (i) Spearman correlation between
$\sigma_c$ and the absolute Monte Carlo cost-to-go prediction error; (ii) area
under the receiver-operating-characteristic curve (AUC), where $0.5$ denotes
chance ranking and $1$ perfect ranking,
for classifying its top-$10\%$ prediction errors; (iii) boundary-hit AUC
when both event classes occur; and (iv) AUC for distinguishing policy actions
from controlled perturbed actions using $\sigma_c$. We first compute each
statistic separately for every trained policy, then report their mean and a 95\%
confidence interval obtained by repeatedly resampling policies. This avoids
mixing transitions from environments and policies with different scales.
Tables~\ref{tab:sigma_correlation}
and~\ref{tab:sigma_bins} report the overall and environment-level results.

\begin{table*}[t]
\centering
\begingroup
\small
\setlength{\tabcolsep}{6pt}
\renewcommand{\arraystretch}{1.06}
\begin{tabularx}{0.72\textwidth}{@{}Xr@{}}
\toprule
\textbf{Statistic} & \textbf{Value} \\
\midrule
Transitions / episodes / policies & 223,607 / 300 / 15 \\
Boundary-hit AUC (eligible policies) & $0.868\;[0.734,0.968]$ (7/15) \\
Spearman($\sigma_c$, step cost) & $0.112\;[-0.001,0.236]$ \\
Spearman($\sigma_c$, absolute Monte Carlo error) & $-0.030\;[-0.168,0.096]$ \\
Top-$10\%$ prediction-error AUC & $0.520\;[0.455,0.589]$ \\
Perturbed-action AUC & $0.568\;[0.528,0.607]$ \\
 
\bottomrule
\end{tabularx}
\endgroup
\caption{
  Transition-level critic-disagreement diagnostic under the common final-evaluation
  protocol.
  Intervals are 95\% confidence intervals obtained by resampling trained
  policies; boundary AUC includes only policies containing both hit and non-hit
  transitions.
}
\label{tab:sigma_correlation}
\end{table*}

\begin{table*}[t]
\centering
\begingroup
\small
\setlength{\tabcolsep}{7pt}
\renewcommand{\arraystretch}{1.06}
\begin{tabularx}{0.82\textwidth}{@{}Xrrrr@{}}
\toprule
\textbf{Environment} & \textbf{Hits} & \textbf{Hit AUC} & \textbf{Error $\rho$} & \textbf{Action AUC} \\
\midrule
Ant & 1920 & $0.911\pm0.092$ (5/5) & $-0.222\pm0.290$ & $0.584\pm0.066$ \\
Hopper & 0 & -- (0/5) & $0.034\pm0.265$ & $0.553\pm0.020$ \\
Goal1 & 1169 & $0.761\pm0.338$ (2/5) & $0.098\pm0.165$ & $0.567\pm0.133$ \\
 
\bottomrule
\end{tabularx}
\endgroup
\caption{
  Environment-level critic-disagreement diagnostic.
  Entries are mean$\pm$standard deviation over five policies. ``Hits'' counts
  boundary-hit transitions; parentheses show the number of policies eligible
  for boundary AUC. ``Error $\rho$'' is the Spearman
  correlation with absolute cost-to-go prediction error, and ``Action AUC'' uses
  disagreement to distinguish policy actions from perturbed actions.
}
\label{tab:sigma_bins}
\end{table*}

Tables~\ref{tab:sigma_correlation} and~\ref{tab:sigma_bins} show a limited
positive result for boundary events: all five Ant policies
and two Goal1 policies contain both event classes, while Hopper has no boundary
hits. The mean boundary AUC over these seven eligible policies is $0.868$, but
the mean prediction-error correlation is $-0.030$ and the top-$10\%$ error AUC
is $0.520$, both statistically compatible with chance. Controlled perturbed
actions induce a modest but consistent spread increase (AUC $0.568$,
95\% CI $[0.528,0.607]$). Thus the ensemble signal can flag specific
boundary-associated and action-perturbed decisions, but it is not a generally
numerically reliable estimator of prediction error, aleatoric tail probability, or policy
safety. This mixed result supports the paper's limited
\emph{critic-disagreement shaping} interpretation and rules out a stronger
claim that disagreement values track true error or risk.
Because the diagnostic observes final-policy trajectories plus small local
action perturbations, it does not show how disagreement behaves in arbitrary
unvisited regions or identify the causal effect of the training-time shaping
term.

\paragraph{Controlled missing-coverage diagnostic.}
We next deliberately remove a defined region from the data used to train the
critics and test whether disagreement identifies that region on separate
episodes. We keep the same 15 policies fixed and train fresh $M=5$ cost-critic
ensembles after dividing episodes without overlap into 70\% for fitting and
30\% for evaluation. The structured condition excludes the upper $20\%$ by
either current-observation norm or the dot product of standardized
state--action features with one fixed random direction. Its matched control
uses the same architecture, initialization, number of transitions, bootstrap
keep probability, and optimizer updates, but draws transitions uniformly from all fitting
episodes. Evaluation is performed only on the held-out episodes. We test both
the same TD target used during training (Eq.~\ref{eq:td_target}) and a direct
Monte Carlo cost-to-go target used solely for comparison.

\begin{table*}[t]
\centering
\begingroup
\small
\setlength{\tabcolsep}{6pt}
\renewcommand{\arraystretch}{1.06}
\begin{tabularx}{0.80\textwidth}{@{}XXrr@{}}
\toprule
\textbf{Shift} & \textbf{Target} & $\boldsymbol{\Delta}$ \textbf{OOD AUC} & $\boldsymbol{\Delta\rho_{\rm err}}$ \\
\midrule
Observation norm & Temporal difference & $0.168\ [0.118,0.218]$ & $0.003\ [-0.066,0.064]$ \\
Random projection & Temporal difference & $0.099\ [0.044,0.153]$ & $0.051\ [0.003,0.102]$ \\
\midrule
Observation norm & Monte Carlo & $0.153\ [0.106,0.201]$ & $0.128\ [0.040,0.243]$ \\
Random projection & Monte Carlo & $0.087\ [0.022,0.148]$ & $0.030\ [-0.054,0.114]$ \\
 
\bottomrule
\end{tabularx}
\endgroup
\caption{
  Paired controlled missing-coverage diagnostic over 15 fixed policies.
  $\Delta$ is structured removal minus matched-random critic retraining; brackets are
  95\% confidence intervals obtained by resampling policies. OOD AUC uses
  $\sigma_c$ to classify the
  predefined omitted region on disjoint episodes; $\rho_{\rm err}$ is Spearman
  correlation with absolute Monte Carlo cost-to-go error.
}
\label{tab:coverage_shift}
\end{table*}

Table~\ref{tab:coverage_shift} shows that, under the same TD target used in
training, observation-norm OOD AUC rises from
$0.585$ in the matched control to $0.753$ after structured removal, and
random-projection OOD AUC rises from $0.571$ to $0.670$. The paired gains are
$0.168$ (95\% CI $[0.118,0.218]$) and $0.099$ ($[0.044,0.153]$), respectively,
showing that omitted training data---rather than sample count alone---raises ensemble
spread under this critic-retraining protocol. The corresponding error-correlation gain
is null for observation norm and small for random projection. Direct Monte Carlo fitting
improves error ranking for observation norm, indicating that shared TD-target
bias can decouple disagreement from realized error. Boundary-AUC changes remain
inconclusive because only six policies contain both event classes. This result
supports $\sigma_c$ as a coverage-sensitive disagreement signal, not as a
tail-event probability, general error estimator, or policy-safety guarantee.

\paragraph{Matched target-noise negative control.}
To distinguish omitted training examples from noisy targets, we use the matched-control
training subset, which still contains examples above the omission threshold,
and add zero-mean TD-target noise at
$\{0,0.25,0.5,1.0\}$ times the clean-target standard deviation. Samples,
initialization, minibatch order, and bootstrap masks are identical across noise
levels; the zero-noise replay reproduces the original matched control exactly.
At the primary noise multiplier of $1.0$, chosen before analysis, structured
removal still exceeds the
noisy control in OOD AUC by $0.126$ (95\% CI $[0.069,0.185]$) for observation
norm and $0.110$ ($[0.060,0.159]$) for random projection, with positive paired
effects for 13 of 15 policies under each score. Ant and Hopper have
within-environment confidence intervals entirely above zero, whereas Goal1
remains inconclusive. Error-ranking effects remain null. Thus disagreement
responds more strongly to omitted training data than to this controlled
target-noise model, but the result is not a complete epistemic--aleatoric
decomposition.

\subsection{Paired Actor-Update Intervention}
\label{app:actor_intervention}

We test whether the detached disagreement term changes the PPO actor update in
its intended direction. The diagnostic uses nine BCPPO checkpoints saved after
100k training steps from Hopper, Button1, and Push1 (three seeds each), with
eight independent 1,024-step rollouts generated by each checkpoint's current
policy. For each
rollout, 75\% of transitions update the actor and the remaining 25\% supply
held-out states for measurement. Treatment and control actors start from
identical parameters and use the same rollout, split, newly initialized Adam
gradient optimizer, minibatch order, PPO epochs, and saved cost-critic ensemble
whose parameters remain fixed;
only the detached disagreement term is enabled or disabled.

On each held-out state, we draw 16 candidate actions from the unchanged
pre-update policy. The frozen ensemble assigns each candidate a predicted mean
cost $\widehat Q_c=\mu_c$ and penalty $\RB$. We define
$\Delta\log\pi=\log\pi_{\rm treatment}-\log\pi_{\rm control}$, so a negative
association between $\RB$ and $\Delta\log\pi$, after controlling for
$\widehat Q_c$, means that the penalty lowers the relative probability of
high-disagreement actions. We also examine the rollout actions used by the
update and actions sampled from each actor after the update; for the latter,
$\Delta\sigma_c$ is treatment minus control mean disagreement. The primary
analysis uses the five checkpoints with a positive saved multiplier
($\lambda>0$). The success rule, chosen before analysis, requires confidence
intervals below zero for both the rank-correlation test and a linear-regression
slope after all variables are standardized to zero mean and unit variance. As
negative controls, four $\lambda=0$ treatment--control pairs and 72 pairs with
the disagreement term disabled in both actors should remain identical.

\begin{table}[t]
\centering
\begingroup
\scriptsize
\setlength{\tabcolsep}{3pt}
\renewcommand{\arraystretch}{1.04}
\begin{tabularx}{\columnwidth}{@{}Xrl@{}}
\toprule
\textbf{Diagnostic} & \textbf{Estimate [95\% CI]} & \textbf{Direction} \\
\midrule
Held-out candidates: rank correlation after accounting for $\widehat Q_c$
& $-0.119\;[-0.216,-0.039]$ & 5/5 negative \\
Held-out candidates: standardized slope after accounting for $\widehat Q_c$
& $-0.110\;[-0.520,0.168]$ & 2/5 negative \\
Rollout update actions: Spearman $(\RB,\Delta\log\pi)$
& $-0.040\;[-0.052,-0.031]$ & 5/5 negative \\
Post-update sampled actions: $\Delta\sigma_c$
& $-1.63\;[-4.24,-0.07]\times10^{-6}$ & 4/5 negative \\
$\lambda=0$ policy-parameter difference
& $0$ (exact) & 4/4 zero \\
Both-disable control parameter difference
& $0$ (exact) & 72/72 zero \\
\bottomrule
\end{tabularx}
\endgroup
\caption{
  Paired actor-update intervention.
  Negative treatment-minus-control values indicate steering away from actions
  assigned a larger $\RB$ by the checkpoint's own frozen ensemble.
  Confidence intervals resample checkpoints after averaging the eight rollouts
  for each checkpoint.
}
\label{tab:actor_intervention}
\end{table}

Table~\ref{tab:actor_intervention} shows that the rank-correlation test and two
additional steering diagnostics have intervals below zero. The standardized
linear-slope interval crosses zero, so the rule chosen before analysis requiring both
candidate-action tests is not met. Exact zeros for every $\lambda=0$ pair and
every both-disable pair confirm the expected implementation behavior. Because
this is one update performed after training with a newly initialized shared
optimizer, it
establishes neither an exact continuation of the original training trajectory
nor a causal reduction in final-policy CVaR.

\FloatBarrier
\section{Training Curves}
\label{app:training_curves}

Figures~\ref{fig:training_performance} and~\ref{fig:training_risk} group the
periodic records retained by the 75 internal PPO-family runs. Curves are
five-seed means; bands are standard errors, computed as the standard deviation
across seeds divided by $\sqrt{5}$. These saved model states describe
training dynamics and do not replace the fixed-20 final evaluation in
Table~\ref{tab:main_core}. Figure~\ref{fig:training_risk}(b) shows BCPPO
disagreement after a 25-record per-run rolling mean. Released source tables
contain every plotted row and the exact aggregation output.

\begin{figure*}[!p]
\centering
\includegraphics[width=0.90\textwidth]{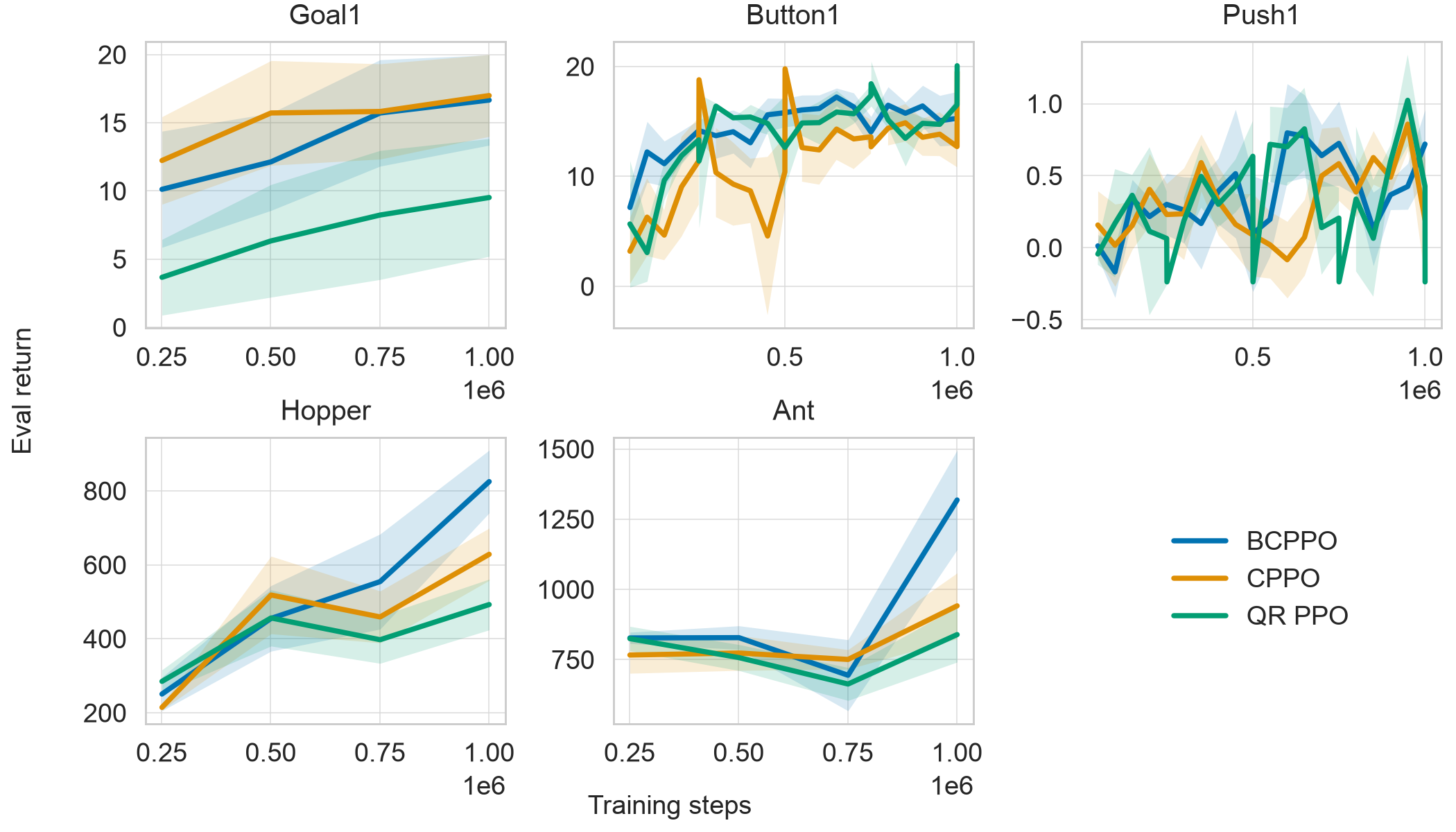}\\[-2pt]
\includegraphics[width=0.90\textwidth]{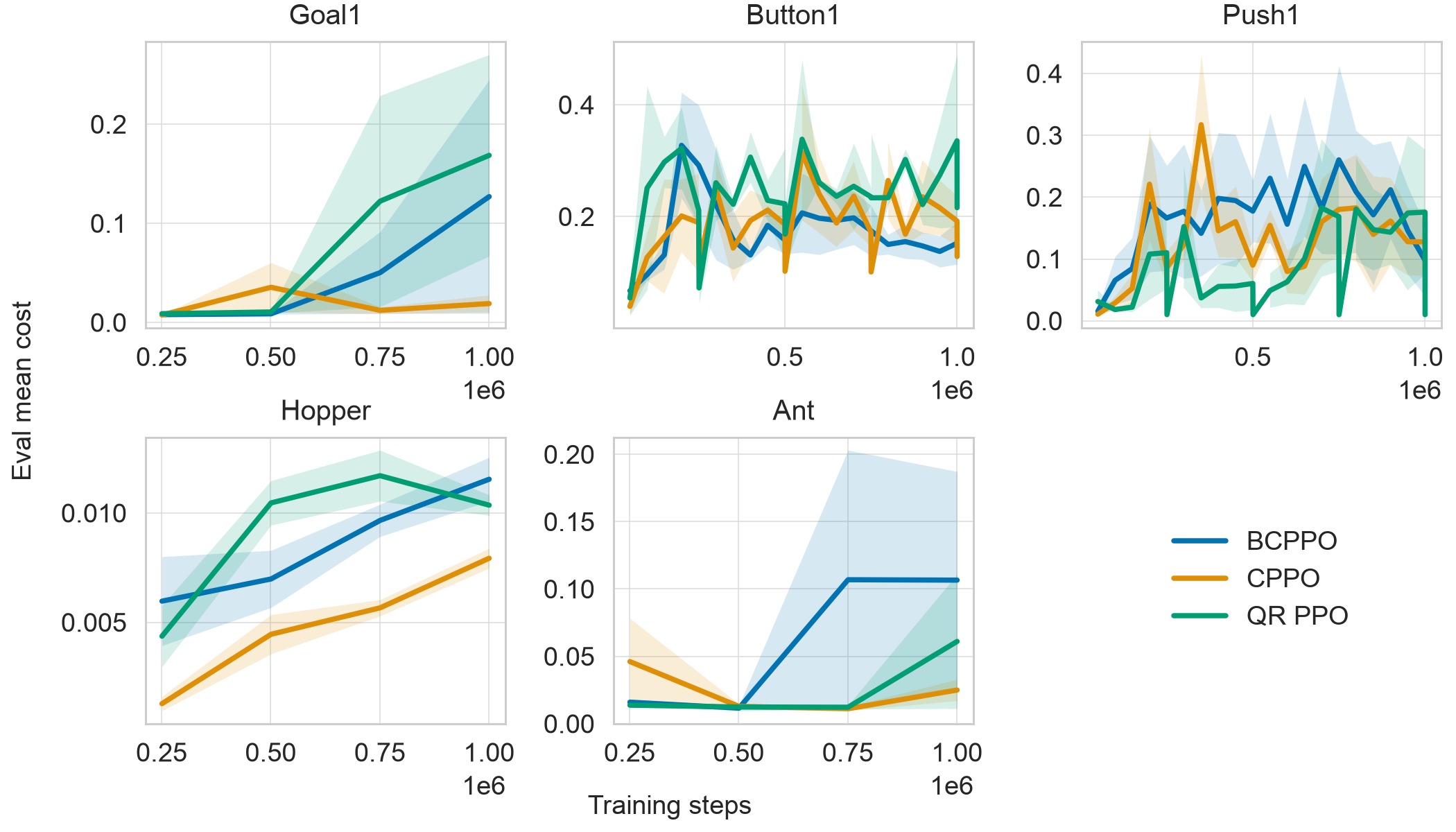}
\caption{Periodic main-experiment performance diagnostics: (a) return and
(b) mean cost. The cost is the episode-average rate used by the main table.}
\label{fig:training_performance}
\end{figure*}

\begin{figure*}[!p]
\centering
\includegraphics[width=0.90\textwidth]{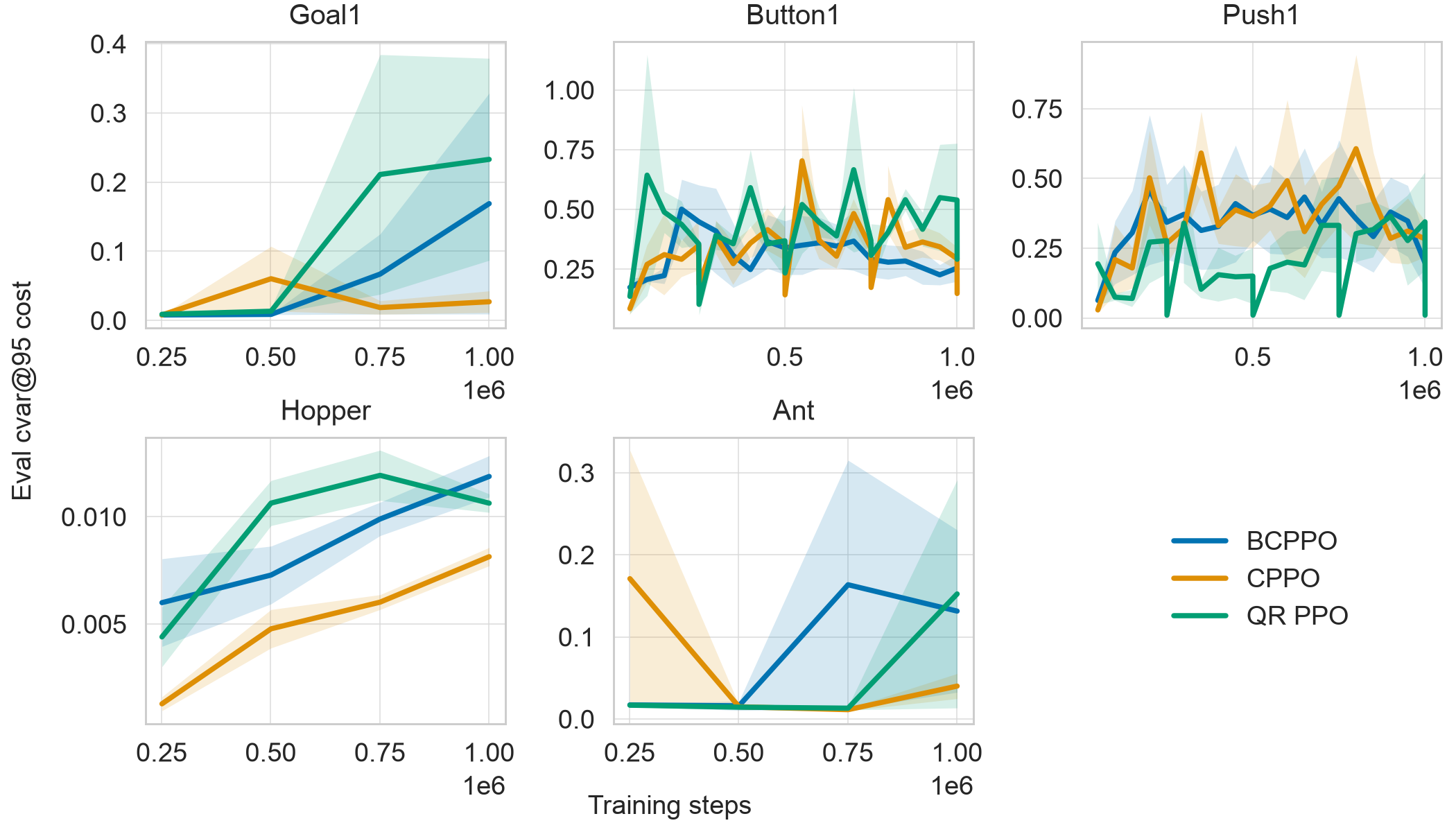}\\[-2pt]
\includegraphics[width=0.90\textwidth]{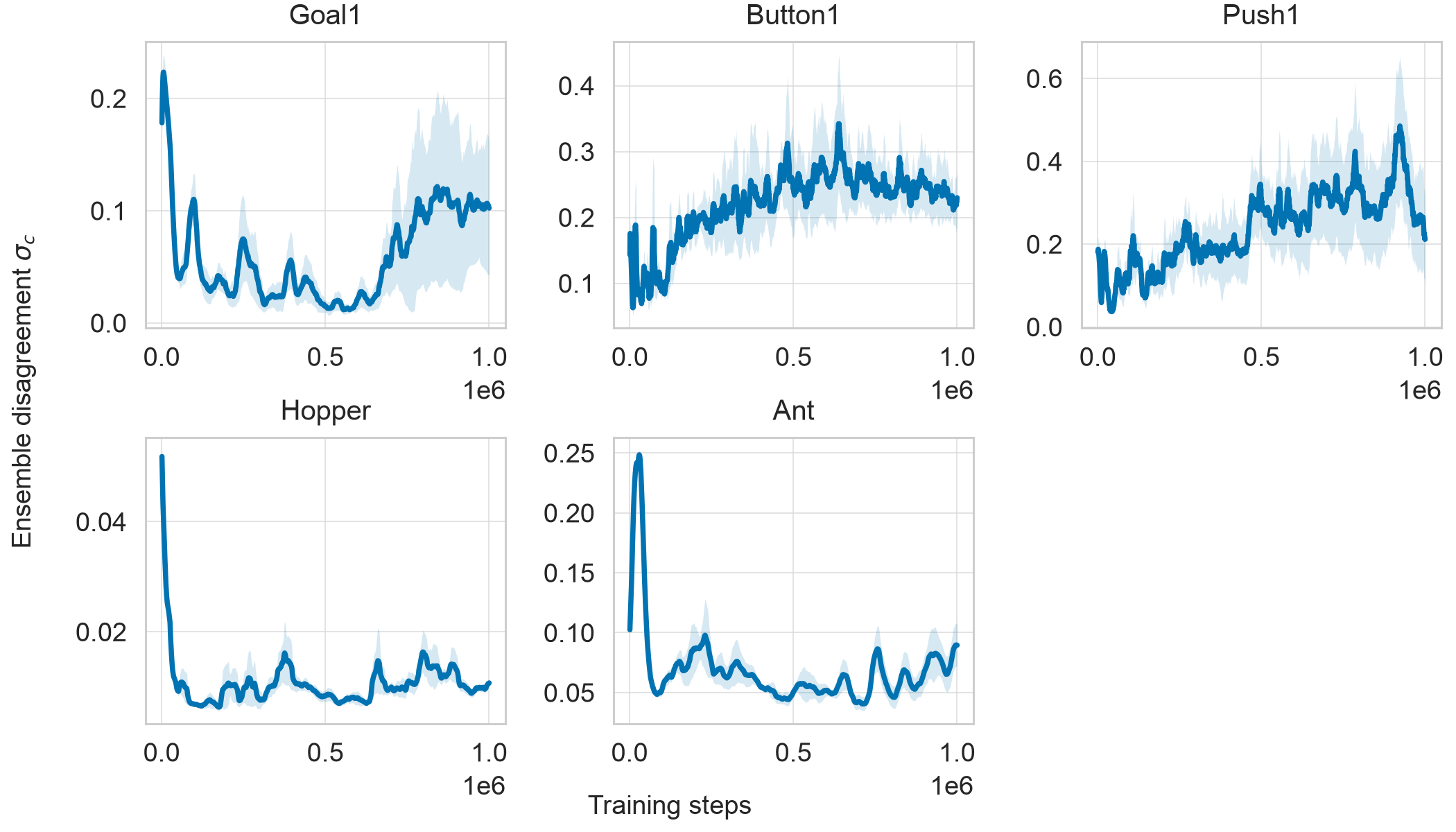}
\caption{Periodic risk diagnostics: (a) empirical CVaR@95 from each
checkpoint's recorded evaluation episodes and (b) BCPPO cost-critic
disagreement. The disagreement scale is environment dependent.}
\label{fig:training_risk}
\end{figure*}

\section{Ablation Studies}
\label{app:ablations}

\subsection{Component and Sensitivity Results}

This appendix provides the full component, uncertainty-placement,
hyperparameter, and ensemble-size diagnostics referenced in
Section~\ref{sec:experiments}. All runs use the same rare-event wrapper and
20-episode final evaluation as the main table, but train for 500k steps over
five seeds on Button1, Push1, and Hopper. Tables~\ref{tab:ablation_components},
\ref{tab:ablation_sensitivity}, and~\ref{tab:bachelier_sigma} report the
component, sensitivity, and scale-control results, respectively.

\begin{table*}[t]
\centering
\begingroup
\small
\setlength{\tabcolsep}{5pt}
\renewcommand{\arraystretch}{1.05}
\begin{tabular}{@{}llccc@{}}
\toprule
\textbf{Variant} & \textbf{Metric}
& \textbf{Button1} & \textbf{Push1} & \textbf{Hopper} \\
\midrule
\multirow{3}{*}{\textbf{\bcppo{}}} & Return & \pmv{16.348}{5.158} & \pmv{0.149}{0.342} & \pmv{851.32}{1160.49} \\
& Cost & \pmv{0.285}{0.256} & \pmv{0.290}{0.262} & \pmv{0.007}{0.005} \\
& CVaR@95 & \pmv{0.414}{0.325} & \pmv{0.515}{0.454} & \pmv{0.007}{0.005} \\
\addlinespace[1.5pt]
\multirow{3}{*}{no Bachelier} & Return & \pmv{15.051}{4.786} & \pmv{0.799}{0.183} & \pmv{534.36}{183.03} \\
& Cost & \pmv{0.216}{0.152} & \pmv{0.305}{0.330} & \pmv{0.009}{0.006} \\
& CVaR@95 & \pmv{0.356}{0.204} & \pmv{0.550}{0.522} & \pmv{0.009}{0.006} \\
\addlinespace[1.5pt]
\multirow{3}{*}{uncertainty bonus} & Return & \pmv{14.266}{1.255} & \pmv{0.069}{0.550} & \pmv{354.22}{219.20} \\
& Cost & \pmv{0.358}{0.175} & \pmv{0.394}{0.339} & \pmv{0.007}{0.004} \\
& CVaR@95 & \pmv{0.541}{0.240} & \pmv{0.843}{0.544} & \pmv{0.007}{0.004} \\
\addlinespace[1.5pt]
\multirow{3}{*}{no branch normalization} & Return & \pmv{17.651}{2.294} & \pmv{0.073}{0.270} & \pmv{520.32}{102.26} \\
& Cost & \pmv{0.122}{0.009} & \pmv{0.171}{0.276} & \pmv{0.011}{0.002} \\
& CVaR@95 & \pmv{0.234}{0.054} & \pmv{0.934}{1.479} & \pmv{0.011}{0.002} \\
\addlinespace[1.5pt]
\multirow{3}{*}{no anti-windup} & Return & \pmv{14.400}{6.874} & \pmv{0.398}{0.438} & \pmv{340.38}{307.46} \\
& Cost & \pmv{0.452}{0.170} & \pmv{0.329}{0.246} & \pmv{0.008}{0.007} \\
& CVaR@95 & \pmv{0.668}{0.165} & \pmv{0.554}{0.378} & \pmv{0.008}{0.008} \\
 
\bottomrule
\end{tabular}
\endgroup
\caption{
  Component ablations at 500k steps over five seeds.
  ``no Bachelier'' sets $\beta=0$; ``uncertainty bonus'' moves disagreement to
  the reward branch; ``no branch normalization'' disables independent normalization;
  ``no anti-windup'' keeps the projected PID update but always accumulates its
  integral state.
}
\label{tab:ablation_components}
\end{table*}

The ablations suggest, but do not statistically establish, three localized
design effects. Removing anti-windup increases Button1 mean CVaR@95 from 0.41
to 0.67; moving uncertainty to the reward-bonus branch increases Push1 mean
CVaR@95 from 0.52 to 0.84; and removing branch normalization yields high Push1
seed variance (CVaR@95 standard deviation $1.48$). None is uniform across all
three tasks.

The no-penalty comparison is especially important for the scope of the central
claim. On Button1, the full method gains 1.30 mean return but increases CVaR by
0.058; on Push1, it reduces CVaR by 0.035 but also reduces return by 0.650; on
Hopper, it improves both means, although variation across seeds is large. The first two are
reward--risk trade-offs rather than evidence of unconditional tail reduction.
This is consistent with the proposed critic-training-sensitivity interpretation: the
spread penalty deliberately adds a conservative bias, which can be useful when
disagreement marks consequential model uncertainty and unnecessarily restrictive
when disagreement is benign. The ablation does not establish which regime holds
at an individual state--action pair.

The $\beta$, $\alpha$, and $\kappa$ rows below all change the same effective
coefficient $\beta_{\mathrm{eff}}=\beta h(\Phi^{-1}(\alpha)+\kappa)$ and should
be read as coefficient sensitivity, not three independent safety mechanisms.
Among the tested settings, $\beta=0.10$ and $\kappa=0.30$ have the lowest
reported mean CVaR. On Push1, $M=1$ has no nontrivial spread signal and the
largest reported CVaR variance, while $M=10$ has nearly the same mean CVaR as
$M=5$ without a uniform reward--risk gain. This does not prove that $M=5$ is
uniquely optimal, and the shared default is one common setting rather than the
empirically best setting for every task.

\begin{table*}[t]
\centering
\begingroup
\small
\setlength{\tabcolsep}{6pt}
\renewcommand{\arraystretch}{1.05}
\begin{tabular}{@{}lccc@{}}
\toprule
\textbf{Setting} & \textbf{Return} & \textbf{Cost} & \textbf{CVaR@95} \\
\midrule
Full BCPPO & \pmv{0.149}{0.342} & \pmv{0.290}{0.262} & \pmv{0.515}{0.454} \\
\midrule
$\beta=0.05$ & \pmv{0.194}{0.350} & \pmv{0.197}{0.234} & \pmv{0.411}{0.409} \\
$\beta=0.10$ & \pmv{0.261}{0.701} & \pmv{0.075}{0.068} & \pmv{0.190}{0.154} \\
$\beta=0.25$ & \pmv{0.298}{0.575} & \pmv{0.265}{0.185} & \pmv{0.735}{0.540} \\
$\beta=0.40$ & \pmv{0.458}{0.540} & \pmv{0.173}{0.130} & \pmv{0.350}{0.276} \\
\midrule
$\kappa=0.00$ & \pmv{0.106}{0.930} & \pmv{0.169}{0.133} & \pmv{0.680}{0.654} \\
$\kappa=0.10$ & \pmv{0.650}{1.116} & \pmv{0.108}{0.103} & \pmv{0.268}{0.176} \\
$\kappa=0.30$ & \pmv{\textminus{}0.115}{0.512} & \pmv{0.082}{0.107} & \pmv{0.194}{0.185} \\
\midrule
$\alpha=0.90$ & \pmv{0.589}{0.361} & \pmv{0.153}{0.188} & \pmv{0.354}{0.426} \\
$\alpha=0.99$ & \pmv{0.386}{0.600} & \pmv{0.197}{0.210} & \pmv{0.479}{0.465} \\
\midrule
$M=1$ & \pmv{0.599}{0.263} & \pmv{0.345}{0.427} & \pmv{1.275}{1.602} \\
$M=3$ & \pmv{0.438}{0.356} & \pmv{0.182}{0.142} & \pmv{0.666}{0.857} \\
$M=10$ & \pmv{0.506}{0.372} & \pmv{0.292}{0.348} & \pmv{0.518}{0.525} \\
 
\bottomrule
\end{tabular}
\endgroup
\caption{
  Effective-coefficient and ensemble-size sensitivity on Push1 at 500k steps
  over five seeds. Defaults are $\beta=0.15$, $\kappa=0.15$, $\alpha=0.95$,
  and $M=5$; the first three sweeps all change
  $\beta_{\rm eff}=\beta h(\Phi^{-1}(\alpha)+\kappa)$.
}
\label{tab:ablation_sensitivity}
\end{table*}

\begin{table*}[t]
\centering
\begingroup
\small
\setlength{\tabcolsep}{6pt}
\renewcommand{\arraystretch}{1.05}
\begin{tabular}{@{}llccc@{}}
\toprule
\textbf{Variant} & \textbf{Metric}
& \textbf{Button1} & \textbf{Push1} & \textbf{Hopper} \\
\midrule
\multirow{3}{*}{\textbf{\bcppo{}}} & Return & \pmv{16.348}{5.158} & \pmv{0.149}{0.342} & \pmv{851.32}{1160.49} \\
& Cost & \pmv{0.285}{0.256} & \pmv{0.290}{0.262} & \pmv{0.007}{0.005} \\
& CVaR@95 & \pmv{0.414}{0.325} & \pmv{0.515}{0.454} & \pmv{0.007}{0.005} \\
\addlinespace[1.5pt]
\multirow{3}{*}{raw $\sigma_c$} & Return & \pmv{14.124}{7.956} & \pmv{0.474}{0.750} & \pmv{610.40}{256.35} \\
& Cost & \pmv{0.176}{0.175} & \pmv{0.245}{0.137} & \pmv{0.010}{0.005} \\
& CVaR@95 & \pmv{0.284}{0.249} & \pmv{0.427}{0.187} & \pmv{0.011}{0.005} \\
 
\bottomrule
\end{tabular}
\endgroup
\caption{
  Unmatched-scale Bachelier versus raw-$\sigma_c$ diagnostic at 500k steps over
  five seeds. The raw-$\sigma_c$ variant keeps the same ensemble, detached cost
  branch, normalization, anti-windup PID controller, and $\beta$, but replaces
  $\RB$ by $\sigma_c$. Since $\RB=h(c_0)\sigma_c$, this intentionally changes
  the effective coefficient by $1/h(c_0)$.
}
\label{tab:bachelier_sigma}
\end{table*}

This diagnostic does not compare two structural risk signals. At the defaults,
$h(c_0)\approx0.01446$, so the same-$\beta$ raw variant is approximately
$69.2\times$ stronger. Its lower mean SafetyPoint costs and lower Button1/Push1
CVaR therefore reflect a different operating scale. With
$\beta_{\rm raw}=\beta h(c_0)$, the two actor penalties are algebraically
identical; no empirical superiority claim is made from this table.

\FloatBarrier
\begin{figure*}[!p]
\centering
\includegraphics[width=0.88\textwidth]{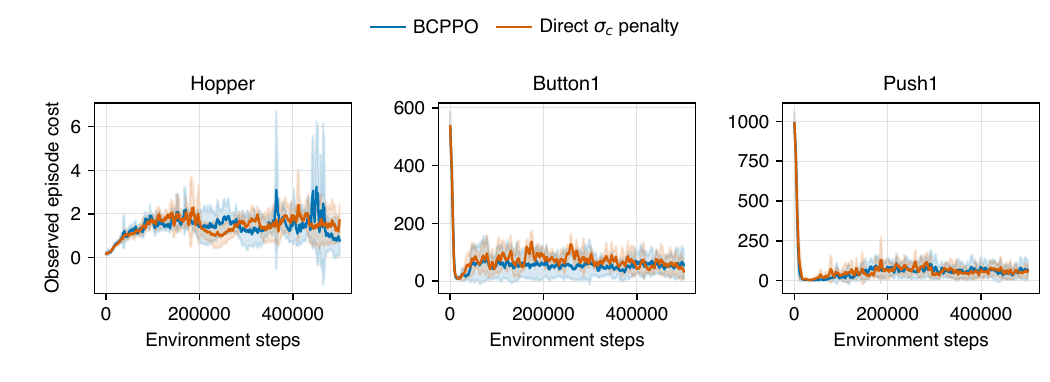}\\[-1pt]
\includegraphics[width=0.88\textwidth]{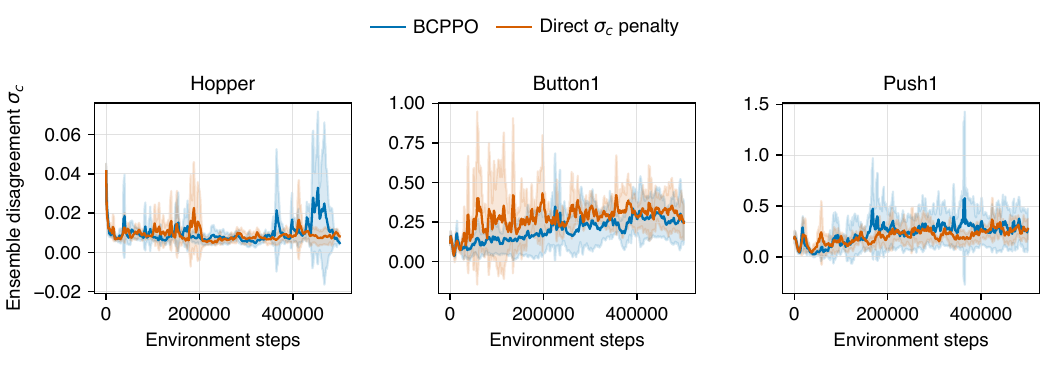}\\[-1pt]
\includegraphics[width=0.68\textwidth]{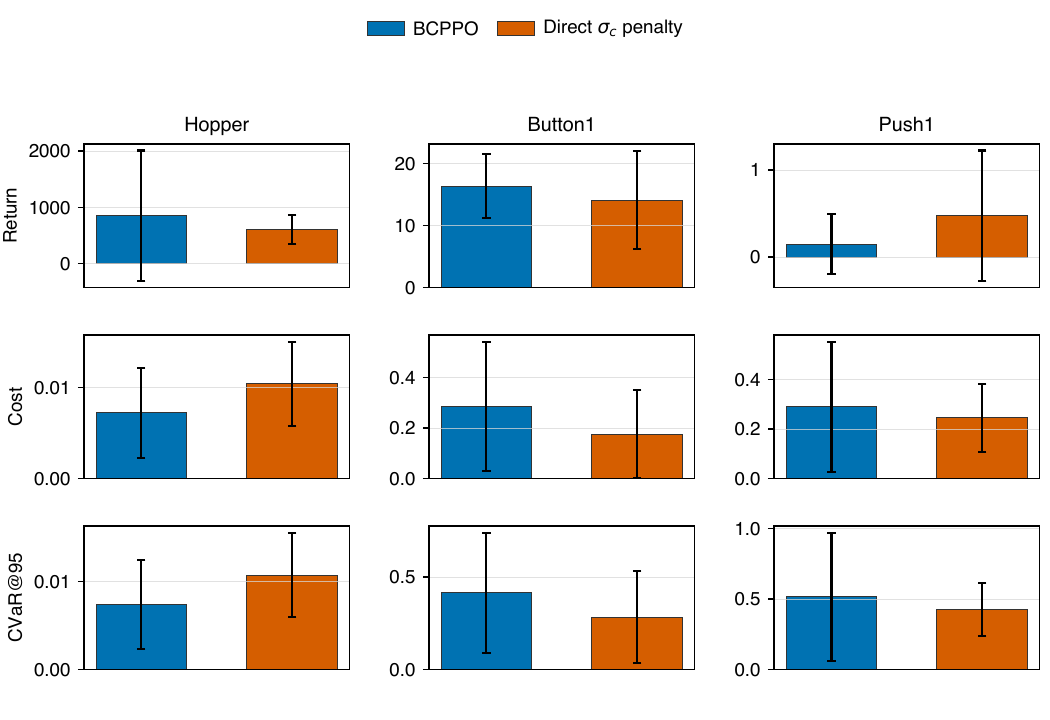}
\caption{Unmatched-scale Bachelier versus direct-$\sigma_c$ diagnostic:
(a) observed cost, (b) ensemble disagreement, and (c) 500k-step final
evaluation. Training curves are five-seed mean$\pm$standard deviation with a
15-record rolling mean; final bars are mean$\pm$standard deviation over five
seeds.}
\label{fig:ablation_bachelier}
\end{figure*}

\begin{figure*}[!p]
\centering
\includegraphics[width=0.72\textwidth]{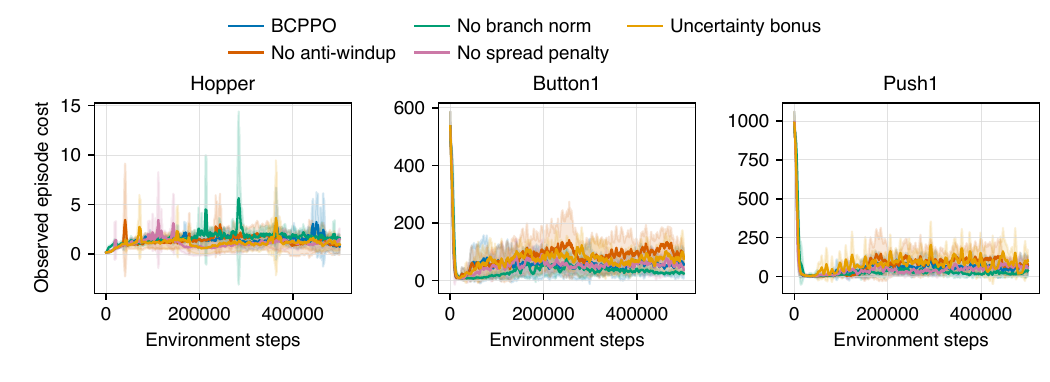}\\[-1pt]
\includegraphics[width=0.72\textwidth]{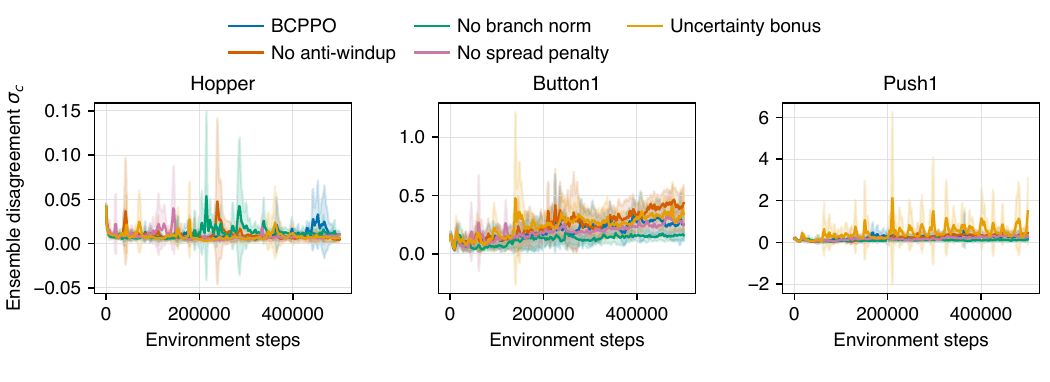}\\[-1pt]
\includegraphics[width=0.72\textwidth]{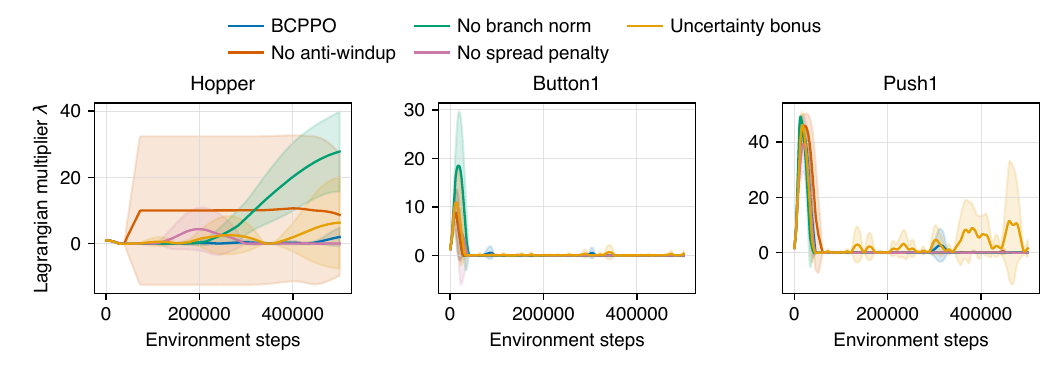}\\[-1pt]
\includegraphics[width=0.54\textwidth]{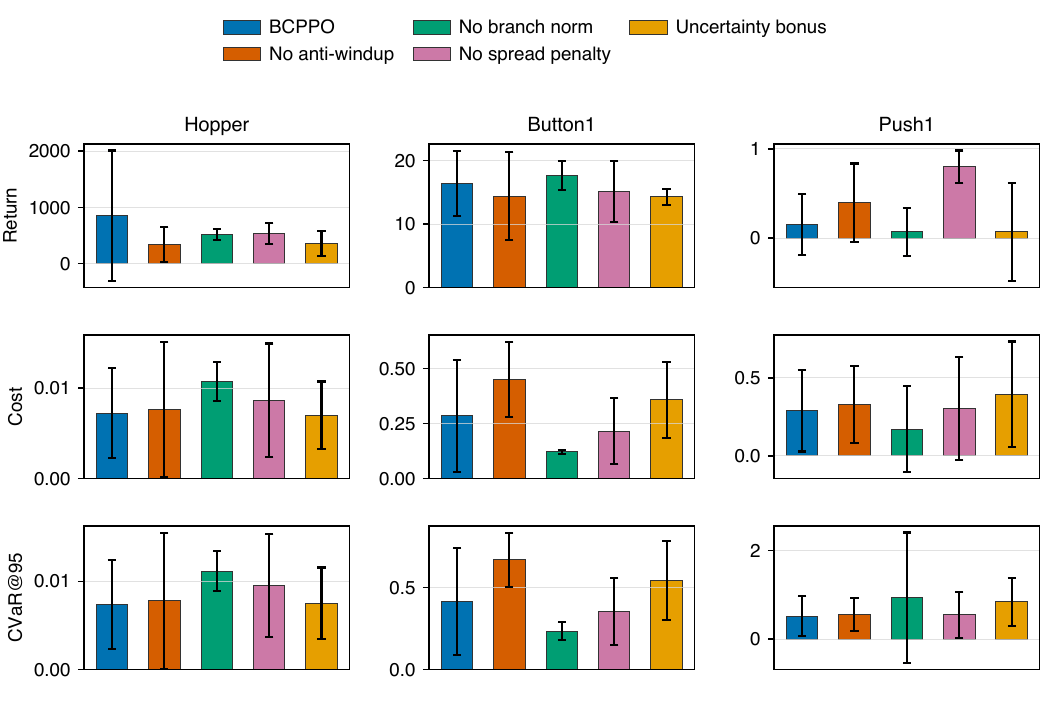}
\caption{Mechanism ablations: (a) observed cost, (b) ensemble disagreement,
(c) multiplier $\lambda$, and (d) 500k-step final evaluation. Training curves
are five-seed mean$\pm$standard deviation with a 15-record rolling mean; final
bars are mean$\pm$standard deviation over five seeds.}
\label{fig:ablation_mechanism}
\end{figure*}

\begin{figure*}[!p]
\centering
\includegraphics[width=0.78\textwidth]{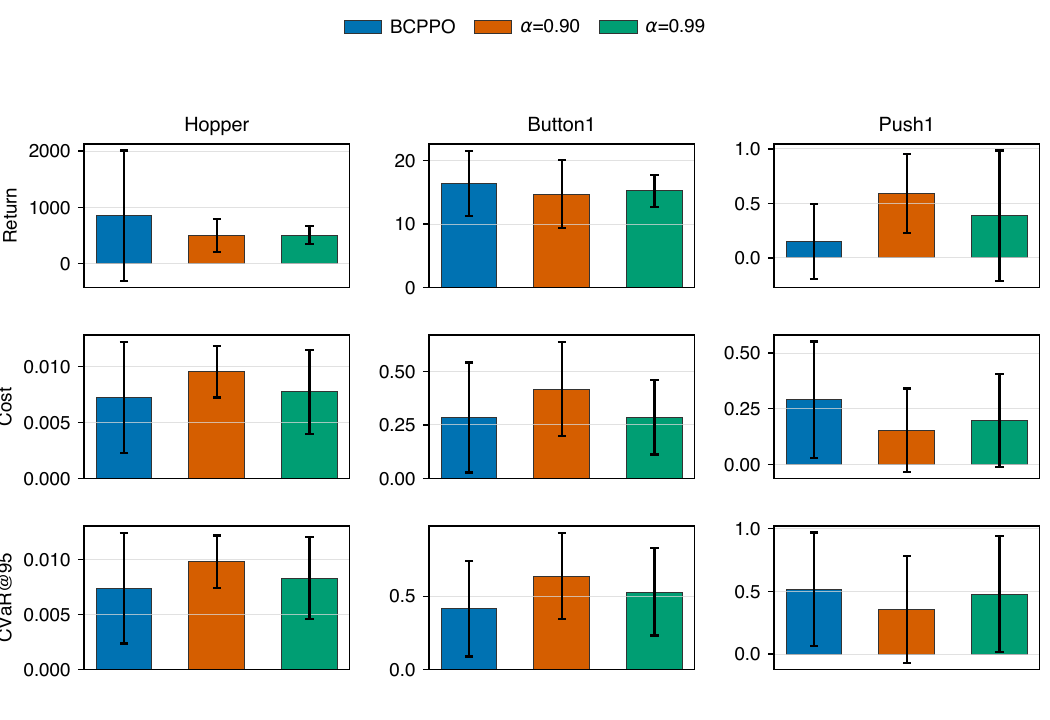}\\[-1pt]
\includegraphics[width=0.78\textwidth]{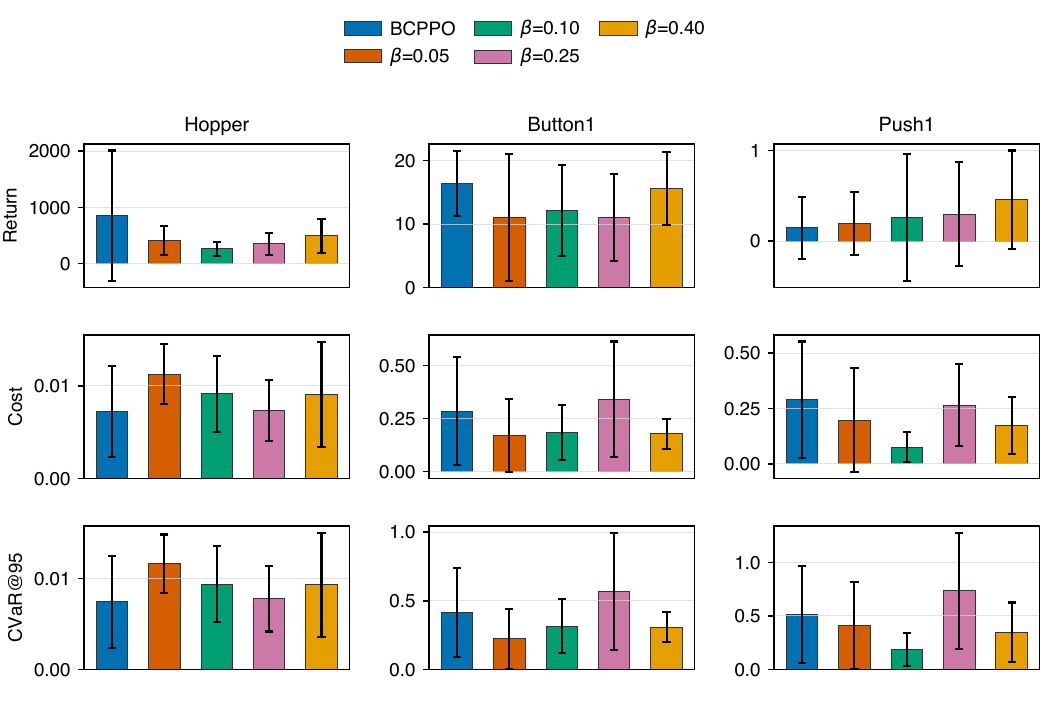}
\caption{Effective-coefficient sensitivity at 500k steps:
(a) $\alpha$ and (b) $\beta$. Bars are five-seed
mean$\pm$standard deviation.}
\label{fig:ablation_sensitivity_ab}
\end{figure*}

\begin{figure*}[!p]
\centering
\includegraphics[width=0.78\textwidth]{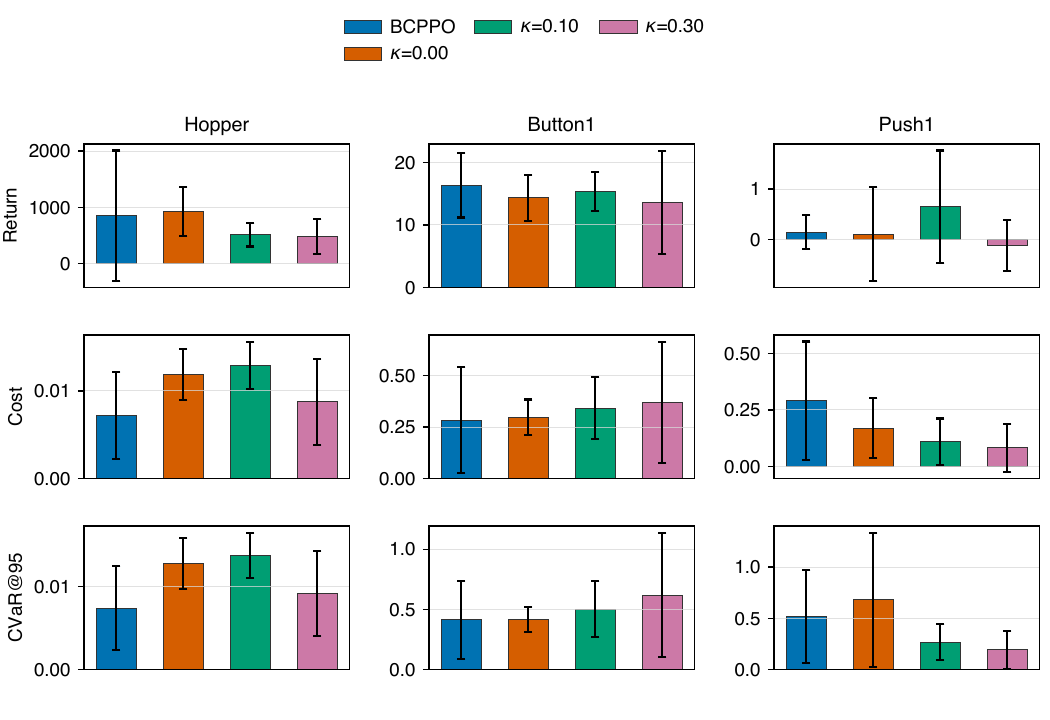}\\[-1pt]
\includegraphics[width=0.78\textwidth]{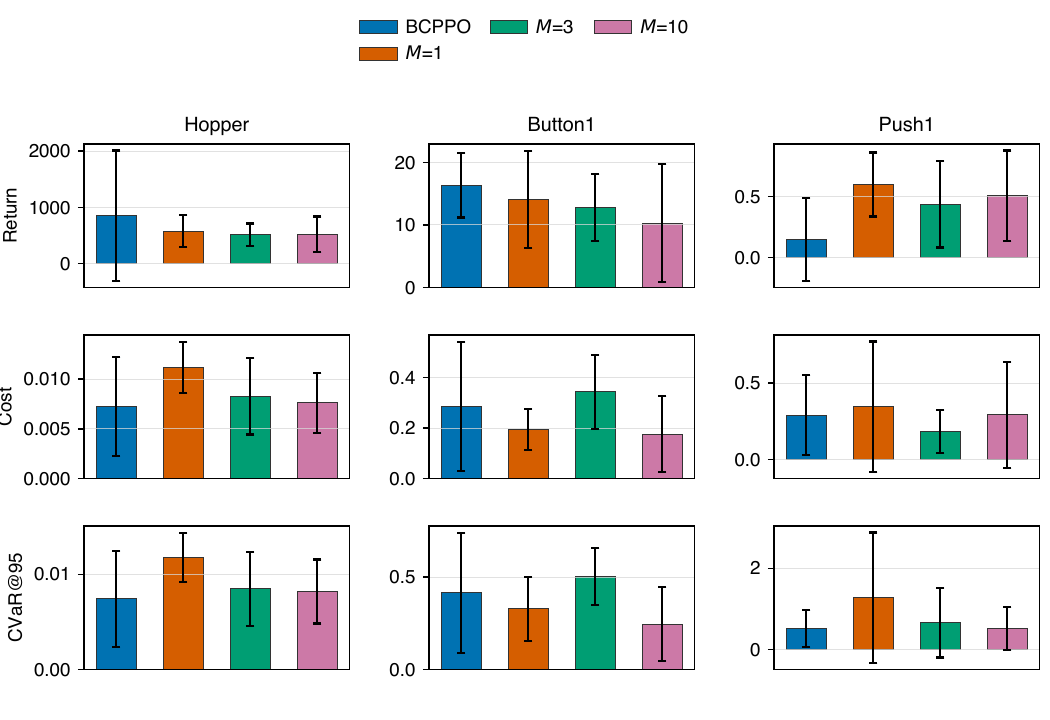}
\caption{Sensitivity at 500k steps: (a) $\kappa$ and (b) ensemble size $M$.
Bars are five-seed mean$\pm$standard deviation.}
\label{fig:ablation_sensitivity_km}
\end{figure*}
\FloatBarrier
\fi

\end{document}